\documentclass{article}
\usepackage{authblk}

\usepackage{natbib}
\usepackage[utf8]{inputenc}
\usepackage{amsmath,amssymb}
\usepackage{graphicx}
\usepackage[utf8]{inputenc}
\usepackage[english]{babel}
\usepackage[toc,page]{appendix}
\usepackage[colorlinks=TRUE, linkcolor=blue, urlcolor=blue, citecolor=blue]{hyperref}
\usepackage[paper=portrait,pagesize]{typearea}
\graphicspath{ {./images/} }
\usepackage[margin=25mm]{geometry}
\usepackage[nocenter]{qtree}
\usepackage{tikz}
\usetikzlibrary{arrows}
\usepackage{breqn}
\usepackage{multirow}
\usepackage{floatrow}
\usepackage{amsthm,bm}
\usepackage{subfig}
\usepackage[belowskip=-9pt,aboveskip=0pt]{caption}
\usepackage{graphicx}
\usepackage{float}
\usepackage[toc,page]{appendix}
\usepackage{tikz}
\usepackage[capitalise]{cleveref}

\newtheorem{definition}{Definition}[section]

\usepackage{float}
\floatstyle{plaintop}
\restylefloat{table}

\newtheorem{lemma}{Lemma}[section]

\usetikzlibrary{bayesnet}

\tikzset{
  treenode/.style = {align=center, inner sep=0pt, text centered,
    font=\sffamily},
  arn_n/.style = {treenode, circle, black, font=\sffamily\bfseries, draw=black,
    fill=white, text width=1.5em}
}

\title{The Best Path Algorithm automatic variables selection via High Dimensional Graphical Models }

\author{ Luigi Riso$^1*$, Maria G. Zoia$^2$, Consuelo R. Nava$^3$}
\date{%
    $^1$ Università  Cattolica del Sacro Cuore, luigi.riso@unicatt.it\\    
    $^2$ Università  Cattolica del Sacro Cuore, maria.zoia@unicatt.it\\
    $^3$ Università  degli Studi di Torino, consuelorubina.nava@unito.it\\
    *corresponding author\\\bigskip
    \today
}

\begin{document}
\maketitle
 \begin{abstract}
This paper proposes a new algorithm for an automatic variable selection procedure in High Dimensional Graphical Models. The algorithm selects the relevant variables for the node of interest on the basis of   mutual information.
Several contributions in literature have investigated the use of mutual information in selecting the appropriate number of relevant features in a large data-set, but most of them have focused on binary outcomes or required high computational effort. The algorithm here proposed overcomes these drawbacks as it is an extension of Chow and Liu's algorithm. 
Once, the probabilistic structure of a High Dimensional Graphical Model is determined via the said algorithm, the best path-step, including  variables with the most explanatory/predictive power for a variable of interest, is determined via the computation of the entropy coefficient of determination. The latter,  being based on the notion of (symmetric) Kullback-Leibler divergence, turns out to be closely connected to the mutual information of the involved variables. 
The application of the algorithm to a wide range of real-word and publicly data-sets has highlighted its potential and greater effectiveness compared to alternative extant methods.
\end{abstract}


\section{Introduction}
The vast availability of large data-sets arises the problem of how  selecting the  explanatory variables of an econometric model. In this case, the use of dimensional reduction techniques \citep{altman2018curse, eklund2007embarrassment} becomes imperative to overcome drawbacks such as multicollinearity and endogeneity of the covariates, or degrees of freedom.
In general, a dimensionality reduction technique can be expressed as an optimization problem \citep{saxena2008dimensionality}, over a $n\times p$ data-set $\textbf{X}$, collecting $n$ observations on  \begin{math} p\end{math} variables \begin{math} \textbf{x}_i, i=1,\dots, p\end{math}. 
The purpose of the dimensional reduction is to obtain a new data-set 
 with \begin{math} {k} \end{math} variables instead of  the original \begin{math} {p} \end{math} ones, that is a data-set 
 with dimensionality \begin{math} {k} \end{math}, where \begin{math} k < p \end{math} and, often, \begin{math} k\ll p \end{math}.
 
 Two basic strategies can be employed to reduce the dimensions from $p$ to $k$: feature selection and feature extraction \citep{khalid2014survey}.
Feature selection (or variable selection) techniques focus on selecting some of the most important features from a data-set 
while feature extraction techniques generate new features which explain either some or the entire data-set information.  Feature selection  is more attractive than feature extraction, because it allows to carry out a more efficient and accurate analysis by eliciting irrelevant information included in a large set of variables \citep{uematsu2019high}, preserving their  interpretability.

The novelty of this paper is to propose an automatic variable selection algorithm, based on Probabilistic mixed Graphical Models \citep{jordan2004graphical}, to detect within a large data-set the relevant variables which have explanatory power toward a specific variable of interest. The Probabilistic High Graphical Model (GM) is obtained by implementing the algorithm devised  by \citet{edwards2010selecting} which is an extension of the Chow-Liu Algorithm. The latter allows to find the probabilistic structure and the maximum likelihood tree in a graph representing the relationships across either discrete or continuous variables of a data-set. The maximum likelihood tree is found by using  mutual information as edge weights on the complete graph with the vertex set representing the variables, and by applying a maximum spanning tree algorithm \citep{gavril1987generating}.  

    Mutual information is a well-established measure of linear and non-linear dependence between variables in information theory \citep[see, e.g.]{steuer2002mutual}, and it can be used to increase understanding of the relationships among their features.

 In literature, some papers have investigated the use of mutual information to select an appropriate number of relevant features  from a data-set \citep{battiti1994using, estevez2009normalized, kwak2002input}. However, most of these algorithms focus only on binary variables, and their implementation requires high computational efforts \citep{kwak2002input}.

Differently, in this paper, the notion of mutual information is used at different steps. At a firs step, it is employed by the \citet{edwards2010selecting} 
algorithm to find the probabilistic relationship between the variables within a high-dimensional graph. At a second step, the reference to mutual information allows to detect the best path of the graph, namely the path including the variables that play a relevant explanatory/predictive role with respect to the variable of interest in the problem at hand. Here, the mutual information between the variable of interest, say $Y$, and the other variables, say $\textbf{X}=X_{1},\dots,X_{i}$, belonging to a given path is measured via the (symmetric) Kullback-Leibler distance between the density of $Y$ and the same  conditioned on $\textbf{X}$, using a procedure suggested by \citet{hall2004cross}. Interestingly, the reference to the Kullback-Leibler distance 
allows to establish a connection with the entropy coefficient of determination of generalized linear models (GLM) \citep{eshima2010entropy,eshima2007entropy}. The latter, which simplifies into the standard determination coefficient (EC) in the case of linear models, provides a statistical measure that can be effectively employed to describe the predictive power of $\textbf{X}$ as potential regressors of a statistical/econometric model explaining $Y$.

One of the advantages of the algorithm here proposed is its computationally efficiency, namely it works well also in a high dimensional contests (i.e. molecular biology, genomics, ect.) \citep[among others]{jones2005experiments, buhlmann2014high, carvalho2008high} where models with hundreds to tens of thousands of variables are employed 
 \citep{edwards2010selecting, meila1999accelerated, kirshner2012conditional}. Furthermore, as the algorithm here proposed selects variables according to given statistical criteria, either EC or $R^{2}$, related to mutual information, it selects variables that play the role of best predictors in a given statistical/econometric model. This, unlike the machine learning approach, it has the advantage of allowing on the one hand the interpretability of the causal links among variables and, on the other hand, the evaluation of the explanatory/predictive performance of the predictors via specific fitting criteria. 
 The remainder of this paper is organized as follows. In Section~\ref{2}, we recall the  properties of the extension of the Chow-Liu Algorithm. In Section~\ref{3}, we present the algorithm and the theory underlying it.  Section~\ref{4} shows  illustrative examples on real data while conclusions follow in Section~\ref{5}.

\section{High-Dimensional Graphical Models}\label{2}
In this section we introduce the essential of High-Dimensional Graphical Models (GM) which proves useful to devise an algorithm for selecting variables in a large mixed data-set.

GMs are used to specify the conditional independence relationships between random variables of a data-set. In a GM these  relationships are depicted as a network of variables. 
More precisely, a graph is a mathematical object \begin{math}
\textbf{G}=(V, E)
\end{math}, where $V$ is a finite set of nodes in a one-to-one correspondence with the random variables of the data-set, while \begin{math}
E \subset V \times V 
\end{math} is a subset of ordered couples of $V$  that defines edges or links representing the interactions between nodes \citep{jordan2004graphical}.

Two generic node, say $u$ and $v$, in a graph $\textbf{G}=(V,E)$, are connected if there is a sequence of distinct nodes, called \textit{path}  \citep{JSSv037i01}, such that $(u=v_1,\dots,v_{i-1}, v_k=v) \in E $, $\forall i=1,\dots,k$.
In the following, we will focus on a mixed data-set, $\textbf{X}$,  
composed of $n$ observations on $p$ variables, among which $r$ discrete, \begin{math} \textbf{D}=(\textit{D}_1,\dots,\textit{D}_r) \end{math} and \textit{q} continuous  \begin{math}
\textbf{C}=(\textit{C}_1,\dots,\textit{C}_q)
\end{math}.
Denote the \textit{i}-observation of \begin{math}
\textbf{X}=(\textbf{D},\textbf{C})\end{math} as \begin{math}
(d_i,c_i), \end{math} with $d_i$ and $c_i$ representing the $i$-th observation of the variables $\textit{D}_i \in \textbf{D}$ and $\textit{C}_i \in \textbf{C}$, respectively.
Given the one-to-one correspondence between variables and nodes, we can write the sets of nodes as $ V=\{ \Delta \cup \Gamma \} $, where $ \Delta$ and $\Gamma$  are the nodes corresponding to the variables in $ \textbf{D}$ and $ \textbf{C}$, respectively. 

One of the central problem in High-Dimensional GMs is the estimation of the underlying probability distributions of the variables from finite samples. 
 \citet{chow1968approximating} proposed an approach for discrete variables which approximates their probability functions via probability distributions of first order tree dependence, that is focusing on  acyclical graphs or trees in which any two vertices are connected by exactly one path. 
Connections between nodes of a tree, or tree dependence, represent the (unknown) joint probability distributions of the nodes \citep{chow1968approximating}, providing information about their mutual dependence or mutual information. 

More in details, \citet{chow1968approximating} found out that a probability distribution of a  tree dependence is an optimum approximation to the true probability of the set of discrete variables composing the tree, if and only if the latter has maximum weight, namely maximum mutual information. 
To gain a better understanding of this result, let $d=(d_1,\dots,d_r)$ be the generic observation or cell of $\textbf{D}$, and let $\mathcal{D}$ be the set of the possible cells or levels of $\textbf{D}$, labelled as  $1,\dots,|D_v|$. 
Assume also that the cell probabilities, $p(d)$, factorize according to a tree $\tau$ -- written as $\textbf{G}_D=(\Delta, E_{\Delta})$, where $\Delta$ and $E_{\Delta}$ are the vertices and the edges set, respectively -- as follows
\begin{equation}
p(d)=\prod_{e \in E_{\Delta}} g_e(d) 
\end{equation}
for a given function $g_e(d)$ depending on the variables included in the edges set $e\in E_{\Delta}$.

If $e=(D_u,D_v)$, then $g_e(d)$ would be only a function of $d_u$ and $d_v$ \citep{edwards2010selecting} and, following \citep{chow1968approximating}, the cell probabilities,  would take the form 
\begin{equation}\label{eq:l1}
    p(\textbf{d}| \tau)= \frac{\prod_{u,v \in E_{\Delta}} p(d_u,d_u) }
    {\prod_{v \in V} p (d_v)^{d_v-1}}=
    \prod_{v \in V} p (d_v)  \prod_{u,v \in E_{\Delta}} \frac{p(d_u,d_u)}{p(d_u)p(d_v)}
\end{equation}
where $d_v$ is the number of edges incident to the node $v$, namely the degree of $v$.

 According to Eq.~\eqref{eq:l1}, the maximized log-likelihood, up to a constant, turns out to be
\begin{equation}
\sum _{(u,v) \in E_{\Delta}} I_{u,v}
\end{equation}
where $I_{u,v}$ is the mutual information between $D_u$ and $D_v$, defined as follows
\begin{equation*}
    I_{u,v}= \sum_{d_u,d_v} n(d_u,d_v) \ln \frac{n(d_u,d_v)}{n(d_u) n(d_v)}
\end{equation*}
 with $n(d_u,d_v)$ denoting the number of observations with $D_u=d_u$ and $D_v=d_v$.

 It is worth noting that the mutual information between two variables is defined as a measure of their closeness \citep{lewis1959approximating}. Therefore, mutual information is a dimensionless, non negative and symmetric quantity which measures the reduction of uncertainty about a random variable, given the knowledge of another.
 Accordingly, \citet{chow1968approximating} developed an algorithm which constructs the optimal dependence tree by maximizing the total branch weights, namely the mutual information, of the latter. Since the branch weights are additive, the maximum weight dependence tree can be constructed brunch by branch. It can be shown that the procedure also maximises the likelihood function and, accordingly, it engenders the maximum-likelihood estimator for the dependence tree. Thus, the maximum likelihood tree for the entire set \textbf{\textit{D}} can be obtained by computing $I_{u,v}$ as edge weights on the complete graph with vertex set $\Delta$, using a maximum spanning tree algorithm \citep{chow1968approximating}. 
 
It can be proved that $I_{u,v}$ is one half the usual likelihood ratio ($LR$) test statistic for marginal independence of $D_u$ from $D_v$, that is
 \begin{equation}\label{eq:test}
 2 I_{u,v}=LR 
 \end{equation}
 which is calculated by using the table of count $\{n(d_u,d_v)\}$ built by cross-tabulating $D_u$ and $D_v$. 
 Under the marginal independence of $D_u$ from $D_v$, $LR$ has an asymptotic $\chi^2_{(k)}$ distribution, with $k$ degrees of freedom, where $k$ 
is the number of parameters considered under the null \citep{edwards2012introduction}.

 A similar approach can be employed to obtain the maximum likelihood tree, $\textbf{G}_C=(\Gamma, E_{\Gamma})$, for a set $\textbf{C}$ of continuous random variables that have a multivariate Gaussian distribution. Here the sample mutual information between two margins $C_u$ and $C_v$ is given by
 \begin{equation*}
     I_{u,v}=-N\ \frac{ln (1-\hat{\rho}^{2}_{u,v})}{2}
 \end{equation*}
where $\hat{\rho}_{u,v}$ is the sample correlation between $C_v$ and $C_u$.\\ As it happens for the case of discrete variables, also in this case the mutual information turns out to be related to the likelihood ratio test statistic $LR$ as in Eq.~\eqref{eq:test}.
Under marginal independence between $C_u$ and $C_v$, the statistic turns out to have a $\chi^2_{(1)}$ distribution \citep{edwards2012introduction}.

Before the appearance of the \citet{chow1968approximating} algorithm, which
has been studied thoroughly, other algorithms were devised to determine the probabilistic structure and the corresponding maximum-likelihood estimator. In particular,  \citet{kruskal1956shortest} provides a simple and efficient solution to this problem. Starting with a null graph, it proceeds by adding at each step the edge with the largest
weight that does not form a cycle with the ones already chosen.
\citet{edwards2010selecting} proposed an extension of the Chow-Liu Algorithm that can be applied with mixed data-set $\textbf{X}$. This algorithm relies on the use of mutual information between a discrete variable, $D_u$, and a continuous variable, $C_v$. It is characterized by  the marginal model which results to be an
ANOVA model  \citep[section 4.1.7][]{edwards2012introduction}.

 It is worth noting that, when dealing with mixed variables, the evaluation of the mutual information $I(d_u,c_v)$ between each couple of nodes requires distinguishing between the case when the variance of $C_v$ is distributed homogeneously across the levels of the discrete variable $D_u$, from the case when it is heterogeneously distributed \citep{edwards2012introduction}. Let $
\{ n_i, \bar{c}_v, s_i^{(v)}\}_{i=1,\dots,|D_u|}$ be the sample cell count, mean and  variance of the couple of variables 
$(D_u, C_v)$.
In the homogeneous case, an estimator of mutual information between $d_u$ and $c_v$ is given by:
\begin{equation*}
    I(d_u,c_v)=\frac{N}{2} \log \left(\frac{s_0}{s} \right),
\end{equation*}
where  
$s_{0}=\sum_{k=1}^{N}(c_{v}^{(k)}-\hat{c}_{v})/N$,  $s=\sum_{i=1}^{|D_u|}n_is_i/N$ and $k_{d_u,c_v}= |D_u|-1$ are the degrees of freedom of the test for marginal independence between the discrete variable $D_u$ and the continuous variable $C_v$.\\
In the heterogeneous case, an estimator of the mutual information is given by \begin{equation*}
    I(d_u,c_v)=\frac{N}{2} \log (s_0) -\frac{1}{2} \sum_{i=1,\dots,|D_s|} n_i \log (s_i)
\end{equation*}
where $
k_{d_u,c_v}= 2(|D_u|-1)$ are the degrees of freedom of the $LR$ test for the marginal independence between the discrete variable $D_u$ and continuous variable $C_v$, with statistic specified as in Eq.~\eqref{eq:test} and with
a $\chi^2_{(k_{u,v})}$ distribution.

 As pointed out by \citet{edwards2010selecting}, one of the disadvantage of selecting a tree on the basis of maximum likelihood is that it always include the maximum number of edges, even if the latter are not supported by data. Thus, he suggested the use of one of the following measures to avoid this drawback
 \begin{equation}
I^{AIC}=I(x_i,x_j)-2k_{x_i,x_j} \end{equation} or
\begin{equation}
I^{BIC}=I(x_i,x_j)-\log(n) k_{x_i,x_j}
\end{equation}
where \begin{math}
k_{x_i,x_j}
\end{math} are the degrees of freedom associated with the pair of variables, $x_i$ and $x_j$, that are defined according to the nature of the variables involved.
The above measures are employed in an algorithm to find the best-spanning tree \citep{edwards2010selecting}. The algorithm stops once the maximum number of edges has been included in the graph.

In the following, we apply the \citet{edwards2010selecting} algorithm which employs penalized likelihood criteria, $I^{AIC}$ and $I^{BIC}$, to detect the maximum likelihood tree in a High-Dimensional mixed GM. As proved by \citet[section 7.4]{hojsgaard2012graphical}, this procedure allows to restrict the analysis to sub-regions of the graph  corresponding to local strongly decomposable models with minimal AIC/BIC \citep{lauritzen1989graphical, edwards2010selecting}. 
This result is guaranteed  by: i) the property of  mixed GMs to be strongly decomposable, as they are acyclic and contain no forbidden paths \citep{lauritzen1989graphical}, that is paths between two non-adjacent discrete vertices passing through continuous vertices \citep[further details in][pp. 7-12]{lauritzen1996graphical}; ii) the property of the \citet{edwards2010selecting} algorithm that avoid the presence of forbidden paths.

\section{The Best Path Algorithm  }\label{3}
To introduce the Best Path Algorithm (BPA), let us consider a data-set $\textbf{D}$ composed of $n$ observations on $p$ mixed variables, $\{X_{1},\dots,X_{p-1}, Y\}$, with $ p\ll n$. 
The BPA first builds  a High-Dimensional GM $\textbf{G}=(V,E)$ for the variables within the data-set. 
Then, the detection of the best set of predictors for the variable $Y$ is carried out by referring to the notion of mutual information. In this regard, the connection between mutual information and entropy, as measured by the Kullback-Leibler (KL) divergence \citep{dembo1991information} is exploited for this scope. Indeed, it is well known that mutual information represents the information gain in one variable due the reduction in its entropy as a consequence of the knowledge of another variable. This gain is related to the KL divergence, that measures the ``distance'' between the distributions of the random variables under exam. The entropy coefficient of determination (ECD) is then employed to select the set of best predictors for the variable $Y$ of interest. ECD is based on the KL divergence between the density of $Y$ and the density of its linear projection on the space engendered by its potential admissible predictors.

In the following, some definitions, useful to develop a strategy finalized to the detection of the admissible predictors for the node of interest, are introduced.
\begin{definition}[Distance]
\label{distance}
Let us assume that a path  exists between the nodes $Y$ and $X_j$, that is
\begin{equation}
Y=X_{j+1},\dots,X_{j+k}=X_j.
\end{equation}
Then, we define distance of order $k$, $d_{Y,X_j}^{(k)}$, the number $k$ of the nodes which are present in the path from $Y$ to $X_j$.
It is simple to note that the following holds
\begin{equation}
    d_{Y,X_j}^{(k)}=n_{(X_{j+1},\dots,X_{j+k})}=d_{X_j,Y}^{(k)}
\end{equation}
where $n_{(X_{j+1},\dots,X_{j+k})}$ is the number of nodes included in the path.
\end{definition}
It is worth noting that  Definition~\ref{distance} does not allow to compute the distance  between nodes belonging to different trees.

\begin{definition}[path-steps]
\label{path-s}
A path-step $w_k$ for the node $Y$, with $k=1,\dots, \max\{d_{Y,X_j}\}$, is the subset of variables $X_i$ whose distance from $Y$ is equal or lower than $k$, that is 
\begin{equation}
  \textbf{X}_{w_k}= \{d_{Y,X_i} \leq k\}.
\end{equation}
The path-step $w_k$ corresponds to the sub-regions of the graph $\textbf{G}=(V,E)$ that satisfies the following property
\begin{equation}
   w_k=\textbf{G}_k=(V,E)_{d_{Y,X_i} \leq k}.
\end{equation}
\end{definition}
The above definitions prove useful to detect the set of variables which play the role of explanatory variables for $Y$ .
For instance, let us consider the graph in Figure~\ref{fig1}  which highlights a minimal BIC forest for a data-set including $p=15$ variables and $n$ observations. The graph is obtained by using the \citet{edwards2010selecting} algorithm.
 In order to detect which subset of the $X_j$  variables with $j=1,\dots,14$ are effective to explain $Y$, we operate an initial selection  by considering only the variables 
composing the tree to which $Y$ belongs. Indeed, insofar as the mutual information, between $Y$ and the variables non included in its tree can be presumed to be approximately null, the hypothesis of marginal independence between $Y$ and these variables can be assumed to hold.
Thus, given the structure of the graph in Figure~\ref{fig1}, the following holds:
\begin{equation}
I(Y, X_{13}) \approx I(Y, X_{12})\approx I(Y, X_{11}) \approx 0
\end{equation}
and, accordingly, the hypothesis of marginal independence between $Y$ and $X_j$ with $j=12,13,14$ turns out to hold \citep{edwards2012introduction}. Consequently, $X_{12}, X_{13}, X_{14}$ can be discarded from the process of feature selection.

\begin{figure}[H]\par\medskip
\centering
\includegraphics[scale=0.60]{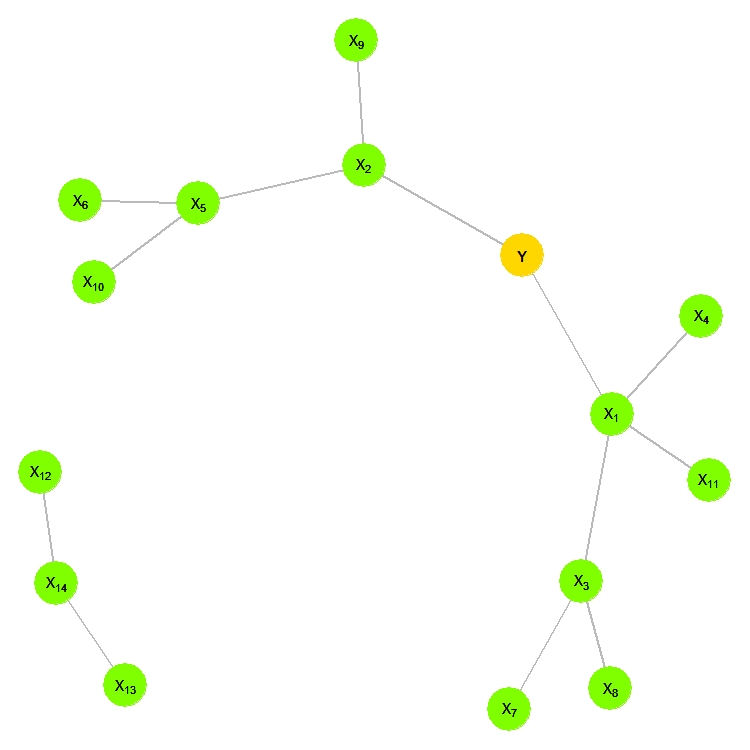}
\caption{Example of a minimal BIC forest from the \citet{edwards2010selecting} algorithm, from a data-set $\textbf{D}$ with $p=15$ variables and $n$ observations}
\label{fig1}
\end{figure}

Let us now consider distances of increasing order between the node $Y$ and the variables $X_{j}$ with $j=1,\dots,11$ belonging to the same tree of $Y$, namely:
\begin{equation}
\label{dis}
\begin{split}
  &d_{Y,X_1}^{(1)}=d_{Y,X_2}^{(1)} ;\\
  &d_{Y,X_3}^{(2)}=d_{Y,X_4}^{(2)}=d_{Y,X_5}^{(2)}=d_{Y,X_9}^{(2)}=d_{Y,X_{11}}^{(2)};\\
  &d_{Y,X_6}^{(3)}=d_{Y,X_7}^{(3)}=d_{Y,X_8}^{(3)}=d_{Y,X_{10}}^{(3)}.
    \end{split}
    \end{equation}
These distances can be used to compute the \textit{path-steps} associated to the node $Y$:
\begin{equation}\label{eq:path}
    \begin{split}
        &w_1=\{ X_1, X_2 \};\\
        &w_2=\{ X_1, X_2, X_3, X_4, X_5, X_9, X_{11} \}; \\
        &w_3=\{ X_1, X_2, X_3, X_4, X_5, X_6, X_7, X_8, X_{10}, X_9, X_{11}\}
    \end{split}
\end{equation}
that include the potential subsets of variables to consider in the feature selection for $Y$.

\bigskip

The following theorem explains how to detect the best path-step among the ones in Eq.~\eqref{eq:path}, where the best path-step is the one  including the variables that are potentially relevant for $Y$.
\begin{lemma}
[The best path-step]
\label{T}
In High-Dimensional Graphical Models $\textbf{G}=(V,E)$ a path-step $w_i$  that satisfies the following property 
\begin{equation}\label{eq:tbp1}
   I(Y, \bm{X}_{ w_i})  \geq  I(Y, \bm{X}_{w_j} ) ,  i \ne j, i=1,\dots,k
\end{equation}
where $ \bm{X}_{ w_{i}}$, $ \bm{X}_{ w_{j}}$ denote the set of variables belonging to the path-steps $ w_{i}$ and $ w_{j}$, always exists for the node of interest $Y$. The path-step $w_i$ is the best one, insofar as it includes the variables that maximize the mutual information with $Y$.
\end{lemma}
\begin{proof}
To prove the existence of a path-step satisfying Eq.~\eqref{eq:tbp1}, let us consider the sum of the mutual information between the variable $Y$ and the set of the variables $\bm{X}=\{X_1,\dots,X_h\}$ belonging to its same tree 
\begin{equation} \label{eq:bp}
    I(Y, \bm{X})= I(Y,X_1)+I(Y,X_2)+ \dots+ I(Y,X_h). 
\end{equation}
Eq.~\eqref{eq:bp} can be rewritten as follows
\begin{equation}
\label{8}
    I(Y, \bm{X}) = \sum_{d_{Y,X}^{(1)}} I(Y,\bm{X})+\sum_{d_{Y,X}^{(2)}} I(Y,\bm{X})+\dots +\sum_{d(Y,X)^{max}}I(Y,\bm{X})
\end{equation}
where $\sum_{d_{Y,X}^{(k)}}$ is the sum of the mutual information between $Y$ and the variables $X_j$ which are $k=1,2,\dots$ nodes distant from $Y$ (see Definition~\ref{distance}).

Now, taking into account that mutual information is non-negative monotone, the following holds for the mutual information between the variables belonging to the different paths steps $w_k$, $k=1,2,\dots$
\begin{equation}\label{eq:tgt}
\sum_{\bm{X} \in w_1 }I(Y,\textbf{X}) \leq \sum_{\bm{X} \in w_2 }I(Y,\bm{X}) \leq \dots \leq \sum_{\bm{X} \in {\max(w_k)}}I(Y,\bm{X}).
\end{equation}
In light of Eq.~\eqref{eq:tgt}, the best path-step, say $w_i$, is the one including  the variables whose mutual information with $Y$ is significantly higher than the mutual information associated with variables included in lower path-steps ($w_j, j<i$), but approximately equal to that of variables included in higher path-steps ($w_j, j>i$), that is
\begin{equation}
\sum_{
\bm{X} \in w_{i-1}} I(Y,\bm{X}) < \sum_{\bm{X} \in w_{i}} I(Y,\bm{X})\approx \sum_{\bm{X} \in w_{i+1}} I(Y,\bm{X})\approx,\dots,\approx \sum_{\bm{X} \in {\max(w)}}I(Y,\bm{X}).
\end{equation}
\end{proof}

\bigskip

\noindent The notion of  path-step, as proposed in Definition~\ref{path-s}, can be alternately expressed in term of the KL Information.
In this regard, let us first recall that KL, which is closely related to relative entropy, or information divergence, measures  the difference between two probability distributions, say $f(Y)$ and $g(Z)$: 
\begin{equation}
KL\left[f(Y),g(Z)\right]=\mathbb{E}_{f(Y)}\left(\log \frac{f(Y)}{g(Z)} \right)
\end{equation}
where $\mathbb{E}_{f(Y)}$ denotes that the expectation is taken with respect to $f(Y)$.

KL is asymmetric in the two distributions, meaning that  $KL[f(Y),g(Z)] \neq KL[g(Z),f(Y)]$, and it does not satisfy the triangle inequality.
Within this research contest, KL can be employed to  measure the divergence between the density of $Y$, say $f(Y)$, and the density of the same variable when conditioned to the set of variables $\bm{X}_{w_{i}}=(X_1,\dots,X_{j},\dots)$ belonging to a given path-step $w_{i}$
\begin{equation}
KL [f(Y), f(Y|\bm{X}_{w_i})]= \sum_{ \bm X \in w_i}\mathbb{E}_{f(Y)}\left(\log\frac{f(Y)}{f(Y|{\bm X}_{w_i})}\right)
\end{equation}
 with
 \begin{equation}
    f(Y)=\sum_{\bm X\in w_i} \frac{f(Y|\bm {X}_{w_i})}{ \{\# \bm X\in w_i\}} 
    \label{fy}
 \end{equation}
 where $\{\# \bm X\in w_i\}$ is the number of  variables belonging to the path-step $w_i$.
 
Now, as well known, the Kullback-Leibler Information of the two above distributions is related to their mutual information, as the latter tallies with the expectation of the former, that is
\begin{equation}\label{eq:kll}
I(Y|\bm{X}_{{w}_{i}})=\mathbb{E}_{f(Y)}[ KL (f(Y), f(Y|\bm{X}_{w_i})].
\end{equation}
The more different the two distributions $f(Y)$ and $f(Y|\bm{X}_{w_i})$, the greater the reduction of the entropy of $Y$ when the set of variables ${X}_{ w_i}$ are employed to explain $Y$. 

Thus, for two sets of explanatory variables $\bm{X}_{1}=(X_1,\dots,X_{i})'$ and $\bm{X}_{2}=(X_1,\dots,X_{q})'$, if it occurs that
\begin{equation}\label{eq:lk}
KL[f(Y), f(Y|\bm{X}_{1})] \approx KL[f(Y), f(Y|\bm{X}_{2})]
\end{equation}
then also the following 
\begin{equation}\label{eq:lk1}
I(f(Y),  f(Y|\textbf{X}_{1}) \approx I(f(Y), f(Y|\textbf{X}_{2})
\end{equation}
must hold. 
Equality in Eq.~\eqref{eq:kll} allows to reformulate the notion of best path-step as follows.
\begin{lemma}\label{path-s1}
In High-Dimensional Graphical Model a path-step $w_i$ that satisfies the following property 
\begin{equation}\label{eq:tbp}
    \sum_{\bm X \in w_i} KL \left[f(Y), f(Y|\bm {X}_{w_i})\right]  \geq    \sum_{\bm X \in w_j} KL[f(Y), f(Y|\bm{X}_{w_j})],   \forall j=1,\dots,k; i \ne j
\end{equation}
always exists for the node of interest $Y$. The path-step $w_i$ is the best path-step, insofar as it includes the variables that maximize the information gain for $Y$. 
\end{lemma}
\begin{proof}
The proof hinges on Lemma~\ref{path-s} and the relationship between KL and mutual information given in Eq.~\eqref{eq:kll}.
\end{proof}
\bigskip

\noindent Once the best path-step in a Higher Dimensional GM has been identified, the best set of predictors for a variable of interest $Y$ can be detected by employing the entropy coefficient, a statistical measure which relies on the notion of Kullback-Leibler divergence. 
Taking advantage of the relationship between the Kullback-Leibler divergence and mutual information, \citet{eshima2007entropy} introduced the entropy coefficient which hinges on the symmetric Kullback-Leibler Information, $KL^{(s)}$ hereafter. 
The entropy coefficient, for measuring the divergence between  $f(Y)$ and $f(Y|\bm X_{w_i})$  is given by 
\begin{equation}\label{eq:EC}
EC=\frac{\sum_{\bm X \in w_i }KL^{(s)}(f(Y), f(Y|\bm X_{w_i}))}{\{\# \bm X\in w_i\}}
\end{equation}
where $KL^{(s)}$ is the symmetric Kullback-Leibler Information between $f(Y)$ and $f(Y|\bm {X}_{w_i})$
\begin{equation}\label{eq:ht}
KL^{(s)}[f(Y), f(Y| \bm X_{w_i})] =KL[f(Y), f(Y|\bm X_{w_i})]+KL[f(Y| \bm X_{w_i}),f(Y)]= 2iI(f(Y),f(Y)|\bm{X}_{w_{i}})
\end{equation}
It is trivial to show that if the path-step $w_1$ contains only one variable, the EC is equal exactly to zero.
Indeed, \citet{eshima2007entropy} proved that Eq.~\eqref{eq:EC} can be used to measure the predictive power of Generalized Linear Models, GLMs hereafter, when the density function of the response variable $Y$, is conditioned to a set of explanatory variables $\bm{X}_{w_{i}}$, belongs to the exponential family:
\begin{equation}\label{eq:mod}
f(y|\bm{X}_{w_{i}})=\exp\left( \frac{y \theta- b(\theta) }{a(\phi)} + c(y,\phi) \right)
\end{equation}
where $\theta$ and $\phi$ are parameters, and $a(\phi)$, $b(\theta)$ and $c(y,\phi)$ specific functions with $a(\phi)>0$. 
They proved that
\begin{equation}\label{eq:cc1}
KL^{(s)}[f(Y), f(Y|\bm{X}_{w_{i}})] = \frac{cov(\theta,Y)}{a(\phi)}=EC
\end{equation}
where $EC=\frac{cov(\theta, Y)}{a(\phi)}$ can be viewed as the average change of uncertainty of the response variable $Y$ as function of the explanatory variable $\bm{X}_{w_{i}}$. 
In the GLM \eqref{eq:mod}, $cov(\theta, Y)$ is nonnegative, and it is zero if and only if X
and Y are independent, i.e. $f(Y)= f(Y|\bm{X}_{w_{i}})$. 
 \citet{eshima2010entropy} demonstrated that $cov(\theta, Y)$ can be interpreted as the the entropy of $Y$ explained by the set of variables $\bm{X}_{w_{i}}$, while $a(\phi)$ is the residual entropy of $Y$, namely the entropy of the latter not explained by $\bm{X}_{w_{i}}$. 
Later on, the authors proposed a slight modification of Eq.~\eqref{eq:cc1}, called entropy coefficient of determination (ECD), given by
\begin{equation}\label{eq:kh}
ECD=\frac{cov(\theta, Y)}{cov(\theta,Y) + a(\phi))}=\frac{EC}{EC+1}
\end{equation}
which can be read as the proportion of the variation in entropy of $Y$ explained by $\bm{X}_{w_{i}}$. 

The reference to either EC or ECD proves useful to determine  the best set of predictors for $Y$. In this regard, we introduce the following lemma.
\begin{lemma}
\label{TT}
In High-Dimensional Graphical Models $\textbf{G}=(V,E)$ a path-step $w_i$  that satisfies the following property 
\begin{equation}\label{eq:tbp11}
  EC(Y, \bm{X}_{ w_i})  \geq  EC(Y, \bm{X}_{w_j} ) ,  i \ne j, i=1,\dots,k
\end{equation}
or equivalently
\begin{equation}\label{eq:tbp12}
  ECD(Y, \bm{X}_{ w_i})  \geq  ECD(Y, \bm{X}_{w_j} ) ,  i \ne j, i=1,\dots,k
\end{equation}
always exists for the node of interest $Y$. The path-step $w_i$ is the best path-step insofar as it includes the variables that maximises the the mutual information with $Y$. 
\end{lemma}
\begin{proof}
The proof rests on the relationship between the symmetric Kullback-Leibler information, on which EC hinges, as expressed by Eq.~\eqref{eq:ht}. The same holds for ECD, in light of Eq.~\eqref{eq:kh} 
\end{proof}
Once, the set of best predictors for the variable $Y$ has been detected, the explanatory role of each of them can be investigated by applying the \citet{kraskov2004estimating}  Mutual Information test of independence between $Y$ and each variable of the optimal pat-step. Indeed, as some variables within $w_{i}$ may be not relevant for $Y$, the set of predictors can be further improved by ruling out those of them which turn out to be not statistically significant for $Y$. This would allow to obtain a more parsimonious set of predictors for $Y$.

In the case of  normal linear models, the set of best predictors for $Y$ can be simple detected by modeling $Y$ with a set of linear regression models, each of them having as explanatory variables the set of variables belonging to the different path-steps, and computing either the coefficient of determination or the adjusted coefficient of determination of each regression. The regression model with the highest coefficient of determination (or adjusted coefficient of determination) identifies the set of best predictors. In this regard, we have the following
\begin{lemma}\label{cor:1}
In case of a normal linear regression model, the set of inequalities in Eq.~\eqref{eq:tbp11} can be replaced by 
\begin{equation}\label{eq:ll3}
\bar{R}^{2}_{Y|\bm{X}_{w_{i-1}}} \leq \bar{R}^{2}_{Y|\bm{X}_{w_{i-1}}} \leq \dots \leq \bar{R}^{2}_{Y|\bm{X}_{ w_{i}}}
\end{equation}
where $\bar{R}^{2}$ is the adjusted coefficient of determination
\end{lemma}
\begin{proof}
In case of a linear regression, $\theta$ is a linear function of the predictors, that is $\theta=\bm{\beta}'\bm{X}_{i}$, and the ECD turns out to be equal to the coefficient of determination $R^{2}$, while $EC=\frac{R^{2}}{1-R^{2}}$.
Given the equivalence of the ECD with the standard coefficient of determination $R^{2}$, and the relationship of the latter with the adjusted coefficient of determination, the best set of predictors for $Y$ is the one that for which $\bar{R^{2}}$ has the highest value.      
\end{proof}
As in the general case, once the set of best predictors for $Y$ is detected, a parsimonious set of the the same can be obtained by 
testing the statistical significance of each of them within the regression model.

The above considerations are at the base of the BPA described here below. This algorithm belongs to Sequential Forward Search given that it starts with an empty set and keeps on adding variables \citep{patil2014dimension}. Then, two distinct algorithms are employed: one for generalized linear models which relies on the entropy coefficient EC to find the best set of predictors for a variable of interest, and the other which is devised for linear models and which makes use of  $\bar{R}^2$ for the choice of the best predictors. The former version is clearly more general as, on the one hand, it can be applied to detect the best set of predictors for either continuous or categorical variables and, on the other hand, being based on densities, it works also when the number of predictors is higher than the sample size. 

The pseudo-code of the algorithm based on the EC is the following:  
 \begin{itemize}
 \item \textbf{Step 0}: run the algorithm to find the optimal tree or forest using the \citet{edwards2010selecting} algorithm, and call this model $\mathcal{M}_0$.
     \item \textbf{Step 1}: select the variable of interest $Y$ and, starting form the node $Y$, identify all path-steps   $w_i$ with $i=1,\dots,k$.
    \item \textbf{Step 2}: $\forall\,i=1,\dots,k$ compute the EC for the  explanatory variables $X_{w_i}$ belonging to each path-step $w_i,  i=1,\dots,k$. \\EC is computed according to Eq.~\eqref{eq:EC}, via the computation of the symmetric Kullback-Leibler divergence between the density of $Y$ and the density of the same conditioned on the set of nodes belonging to the path-steps $w_i,  i=1,\dots,k$. To this aim, the \citet{hall2004cross} algorithm, based on cross-validation and kernel functions to estimate conditional densities, is used. Cross-validation automatically identifies the variables of the path which are relevant for a variable of interest and remove the others. Irrelevant variables are identified by smoothing parameters diverging to their upper extremities as well as by bandwidths diverging to infinity. Cross-validation shrinks these variables to the uniform distribution on the real line and chooses appropriate smoothing parameters and bandwidths for the relevant variables that remain. 
     \item \textbf{Step 3}: pick as best path-step, the one with the highest EC, and call it $ \mathcal{M}_w$. If there were path-steps with equivalent ECs, the most parsimonious one (the one with fewer variables) would be chosen as optimal.
     \item \textbf{Step 4}: implement the \citet{kraskov2004estimating}  Mutual Information test of independence between $Y$ and $X_j \in \mathcal{M}_w$  in order to select the subset $\mathcal{M}_{w_f}$ of variables which are statistically significant for $Y$.
 \end{itemize}
 The pseudo-code of the algorithm  based on the $\bar{R}^2$ is the following:  
 \begin{itemize}
 \item \textbf{Step 0}: run the algorithm to find the optimal tree or forest using the \citet{edwards2010selecting} algorithm, and call this model $\mathcal{M}_0$.
     \item \textbf{Step 1}: select the variable of interest $Y$ and, starting from the node $Y$, identify all path-steps   $w_i$ with $i=1,\dots,k$.
     \item \textbf{Step 2}:   $\forall\,i=1,\dots,k$:
         \subitem (a) fit an econometric  linear regression model explaining $Y$ through $\bm{X}_{w_i}$ and compute $\bar{R}^2$;     
         \subitem (b) implement cross-validation.
     \item \textbf{Step 3}: pick as best model the one with the highest $\bar{R}^2$, and call it $ \mathcal{M}_w$. If there were path-steps with equivalent $\bar{R}^2$, the most parsimonious one (the one with fewer variables) would be chosen as optimal.
     \item \textbf{Step 4}: select only the statistically significant variables based on a t-test in $\mathcal{M}_w$ and call the reduced model  $\mathcal{M}_{w_f}$.
 \end{itemize}
\noindent Both the algorithms operate a reduction of the optimal predictors in the various steps. Indeed, the following holds:
 \begin{equation}
 \mathcal{M}_{w_f} \subseteq \mathcal{M}_w \subset \mathcal{M}_0.
 \end{equation}
 These algorithms are particular cases of the mRMRe algorithm  \citep{de2013mrmre}, originally proposed by \citet{battiti1994using}, which assesses the role played by some variable in explaining another one on the basis of their mutual information with the latter   \citep{kratzer2018varrank}.

 In particular, with reference to the BPA, \textbf{Step 0} operates an unsupervised learning, finalized to highlight the relationships between variables inside the data-set \citep{edwards2010selecting}. At \textbf{Step 1}, the algorithm organizes the variables in specific subsets or path-steps according to a specific variable of interest. At  \textbf{Step 2}, the models with explanatory variables included in the path-steps devised in the previous step are estimated and a cross validation analysis is carried out to avoid problems such as overfitting and selection bias. This also guarantee the possibility to  extend  model results to other independent data-set  \citep{cawley2010over}.
 In this stage the ability of the model to predict data not used in the estimation phase, is checked by dividing randomly the set of observations into $h$ groups (or fold) of approximately equal size. Each fold, in turn, is treated as the validation set \citep{friedman2017elements}. In particular, if the first path-step $w_1$ includes only one variable, the search for the best path-step starts from the second path-step $w_2$.
 At \textbf{Step 3}, the algorithm 
identifies the best path-step by using the Entropy Coefficient of Determination (or $\bar{R}^2$ in case of linear regression models).
 Finally, if some variables of the best path are not statistically significant for $Y$, they are removed in \textbf{Step 4} by applying the \citet{kraskov2004estimating} Mutual Information test.
 
 It is worth noting that the best path-step $w_i$, including the  most relevant variables for explaining $Y$, is chosen not only within all other path-steps but also within all possible subsets of variables included in the data-set.

\section{The Best Path Algorithm at Work}\label{4}
As test-bed for the BPA, some applications employing different real-world and publicly available data-sets are considered.
This, with the aim to show the potentialities of the proposed approach and the breadth of possible multidisciplinary implementations of the algorithm. Either the version of the BPA based on the EC index or the coefficient of determination are considered. 
The first data set, available in R, includes information from Major League Baseball Data from the 1986 and 1987 seasons and it is labelled Hitters \citep{james2013introduction}. The second data-set, still available in R, includes the information taken from a pedometer over a period of nearly 11 months.
The third data-set, created by Athanasios Tsanas of the University of Oxford, includes clinical information from speech signals which were provided by LSVT Global, a company specialising in voice rehabilitation \citep{tsanas2013objective} \footnote{Source: 
\href{https://archive.ics.uci.edu/ml/datasets/LSVT+Voice+Rehabilitation}{https://archive.ics.uci.edu/ml/datasets/LSVT+Voice+Rehabilitation}}
The fourth data-set refers to a study on prostate cancer  \citep{stamey1989prostate}.
The fifth   application employs Communities and Crime data-set\footnote{Source: \href{https://archive.ics.uci.edu/ml/datasets/communities+and+crime}{https://archive.ics.uci.edu/ml/datasets/communities+and+crime}}. It combines socio-economic data from the 1990 US Census, law enforcement data from the 1990 US LEMAS survey, and crime data from the 1995 FBI UCR \citep{redmond2002data}. 
Eventually, the Blog-Feedback data-set is employed  \citep{buza2014feedback} \footnote{Source:
\href{https://archive.ics.uci.edu/ml/datasets/BlogFeedback}{https://archive.ics.uci.edu/ml/datasets/BlogFeedback}}.\\
 The best path-step algorithm based on the EC index is applied on the first three data-sets (Hitters, Steps and LSVT Voice Rehabilitation), to select the appropriate subset of predictors for explaining a given variable, while the version of the algorithm based on $\bar{R}^2$ is performed to select the appropriate set of regressors of a linear model for the other three data-sets .   
The effectiveness of the BPA based on the EC index has been assessed by comparing its results with those obtainable by employing the variable ranking approach employing the mRMRe algorithm. The latter is an appealing filter approach based on both mutual information and redundancy maximum relevance (mRMRe)\citep{battiti1994using, kwak2002input, kratzer2018varrank}.
The prediction performance of the variables selected via the best path-algorithm based on $\bar{R}^2$  has been compared with that of the predictors  detected by employing the Elastic net method. The latter, combining  the ridge  \citep{gruber2017improving} and the lasso method, \citep{tibshirani1997lasso}, is a regularized regression which performs variable selection  \citep{zou2005regularization}.
\newpage

\subsection{Hitters data-set}
This data-set registers $263$ observations for $20$ variables of which three of them are dichotomous while the others are continuous. The BPA based on EC is here employed to find the best appropriate subset of variables that is capable to explain the baseball players' salary in 1987, according to multiple performance statistics registered in 1986. Thus, the algorithm  is implemented for the node associated to the ``Salary'' variable\footnote{1987 annual salary on opening day in thousands of dollars.}. 
Figure~\ref{fig2} shows the conditional relationships among the variables of the data-set identified by the algorithm.

\begin{figure}[H]\par\medskip
\centering
\includegraphics[scale=0.60]{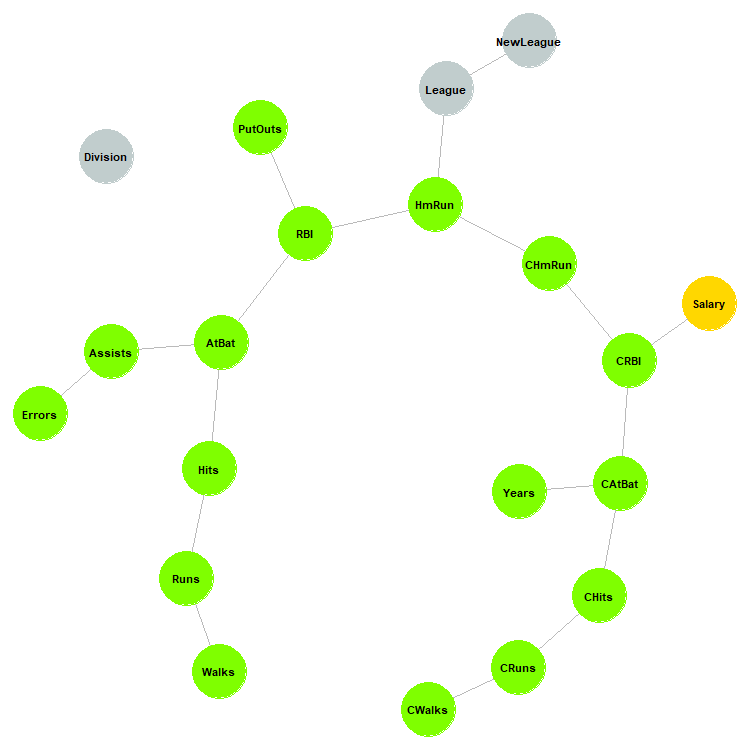}
\caption{The minimal BIC forest for the ``Salary'' variable in the Hitters data-set. Grey nodes denote discrete variables, green nodes denote continuous variables while the yellow node indicates the variable of interest}
\label{fig2}
\end{figure}
Figure~\ref{fig2} shows that the optimal forest is composed of two components: the isolated ``Division''\footnote{A factor with levels E and W indicating player's division at the end of 1986} node, and another big component including all the other variables of the data-set. The ``Salary'' node is directly connected with the ``CRIB'' node, representing the number of runs batted in during a players career. 
Distances and  path-steps  are computed starting from the ``Salary'' node (see Definitions~\ref{distance} and~\ref{path-s}, respectively).
Since the first path-step, $w_1$, turns out to be  composed only of one variable, and the computation of EC  requires more than one regressor \citep{eshima2007entropy}, the search of the best path-step starts form the second path-step $w_2$.  
It results that there are eight possible path-steps that must be explored to find the one with the most explanatory power for ``Salary'' and, among them,  the distance between ``Salary'' and ``Walks''\footnote{Number of walks in 1986.}, equal to $d^{(8)}$, is the maximum one, 

\begin{table}[ht] 
\caption{path-steps' Entropy Coefficient for ``Salary'' in Hitters}  
\centering        
\begin{tabular}{c c }  
path-steps $w_i$ &   Entropy Coefficient (EC) \\  [0.5ex] 
\hline  \hline
$w_1$ & -   \\    
$w_2$ & 0.0192\\ 
$w_3$ & 0.0393\\ 
$w_4$ & 0.0438\\
$w_5$ & 0.0449\\ 
$w_6$ & 0.0442\\
$w_7$ & 0.0425\\
$w_8$ & 0.0423\\  [1ex] 
\hline   \hline   
\end{tabular} 
\label{tab1}  
\end{table} 
Table~\ref{tab1}, shows the Entropy Coefficients for each path-step and the algorithm  identifies
$w_5$ as the best one, given that it has the highest EC coefficient.  Table~\ref{tab2} reports the variables  belonging to the optimal path-step and shows the result of the Kraskov's test, which assesses the independence of each variable belonging to $w_5$ with the variable ``Salary". Indeed, once the best path-step has been detected  using the $EC$ index, the Kraskov's test is implemented to select - within this optimal path-step- only the variables which have a significant role for ``Salary".  
In the following, $w_{f}$ will denote the final set of significant variables included in the best path-step. 

\begin{table}[ht] 
\caption{Kraskov's mutual information test of independence between ``Salary" and $X_{w_5}$. The null of the test assumes independence between each variable in $w_5$ and the varaible ``Salary". } 
\centering        
\begin{tabular}{c c }  
Variables at  $w_5$ &   $p-$value \\  [0.5ex] 
\hline  \hline
CRBI     & 0.01 \\
CAtBat   & 0.01 \\
CHmRun   & 1.00 \\
Years    & 0.98 \\
CHits    & 0.01 \\
HmRun    & 1.00 \\
CRuns    & 0.01 \\
RBI      & 1.00 \\
League   & 0.18 \\
CWalks   & 0.01 \\
AtBat    & 1.00 \\
PutOuts  & 0.84 \\
NewLeague & 0.75\\  [1ex] 
\hline   \hline   
\end{tabular} 
\label{tab2}  
\end{table} 

Figure~\ref{fig3} shows the empirical densities of the variable ``Salary'' conditioned on each variable included in $w_{5}$ which, according to the Kraskov's test, proves to be significant for its explanation, together with  the ``Salary'' density obtained by conditioning the same on the set of variables belonging to $w_{f}$. This density, denoted $f(Salary|w_f$), is computed as in Eq.~\ref{fy}. Figure~\ref{fig4} compares the ``Salary'' empirical density with  $f(Salary|w_f)$.

\begin{figure}[H]\par\medskip
    \centering
    \includegraphics[scale=0.65]{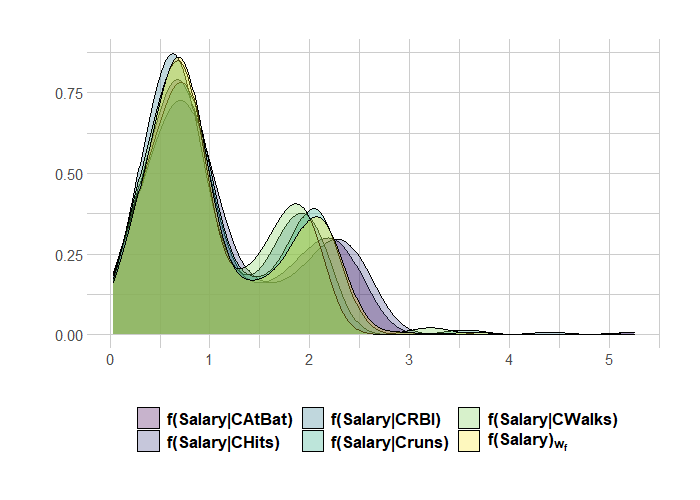}
    \caption{Empirical densities of  ``Salary'' conditioned on  each variables belonging to $w_f$ together with  $f(Salary|w_f)$}
    \label{fig3}
\end{figure} 
\begin{figure}[H]\par\medskip
    \centering
    \includegraphics[scale=0.65]{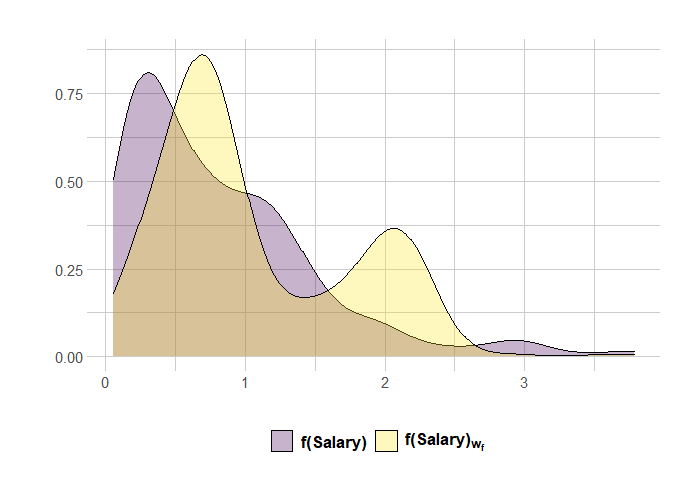}
    \caption{Empirical density of ``Salary'' and  $f(Salary|w_f)$ }   
    \label{fig4}
\end{figure}

\subsection{Steps data-set}
The data-set is composed of $331$ observations and $78$ variables, which are the information obtianed from a pedometer, such as the variable ``Calories", consumed during different activities, the distance covered during a walk, the fats burned and so on. The data-set is available in R within the library("‘TeachingDemos")\footnote{Source:
\href{https://cran.r-project.org/web/packages/TeachingDemos/index.html}{https://cran.r-project.org/web/packages/TeachingDemos/index.html}}. 
The BPA, based on the EC index, is here employed to find the best
appropriate subset of variables that is capable to explain the ``Calories" consumed during the 11 months of pedometer activation. Figure~\ref{fig5} shows the minimal BIC forest for the Steps data-set. The forest is composed of several isolated nodes and one big component including the node of interest ``Calories". By computing distances and path-steps starting from the ``Calories" node, in accordance to Definitions~\ref{distance} and~\ref{path-s},  it emerges that there are 20 potential subsets of variables to consider for feature selection. According to Table~\ref{tab3}, that  shows the Entropy Coefficients for each path-step, $w_2$ is the best one given that it has the highest  EC coefficient.
\begin{figure}[H]\par\medskip
\centering
\includegraphics[scale=0.65]{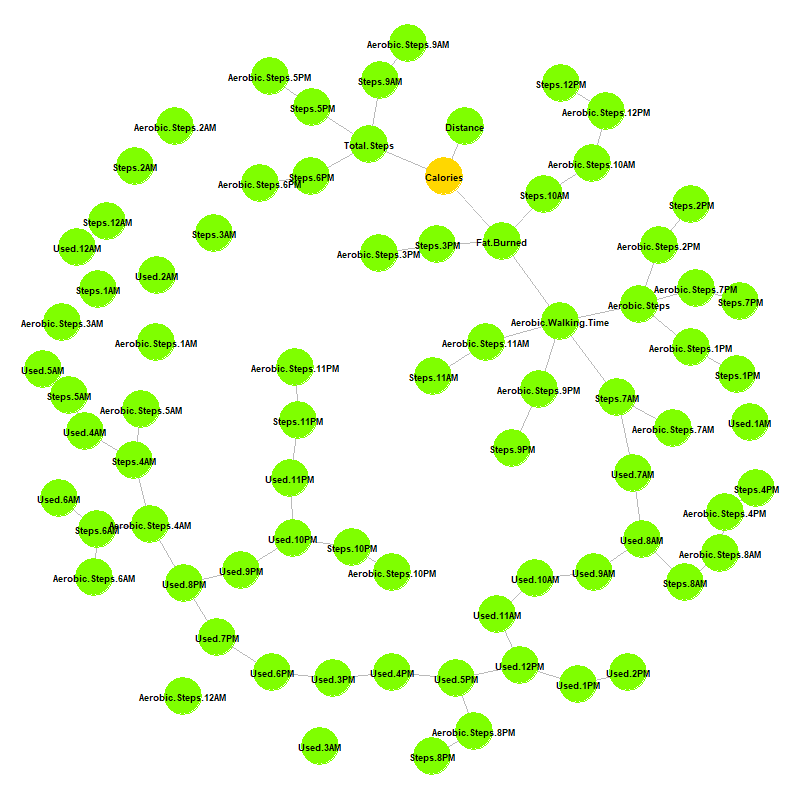}
\caption{The minimal BIC forest for the Steps data-set. Continuous variable are represented by green nodes while the variable of interest is the yellow node. }
\label{fig5}
\end{figure}

\begin{table}[ht] 
\caption{path-steps' Entropy Coefficient for ``Calories'' in Steps}  
\centering        
\begin{tabular}{c c c c }  
path-steps $w_i$ &   Entropy Coefficient (EC) & path-steps $w_i$ &   Entropy Coefficient (EC)\\  [0.5ex] 
\hline  \hline
$w_1$	&	0.0260	&	$w_{11}$	&	0.0778	\\
$w_2$	&	0.1643	&	$w_{12}$	&	0.0751	\\
$w_3$	&	0.1385	&	$w_{13}$	&	0.0736	\\
$w_4$	&	0.1139	&	$w_{14}$	&	0.0723	\\
$w_5$	&	0.1011	&	$w_{15}$	&	0.0709	\\
$w_6$	&	0.0970	&	$w_{16}$	&	0.0690	\\
$w_7$	&	0.0931	&	$w_{17}$	&	0.0666	\\
$w_8$	&	0.0892	&	$w_{18}$	&	0.0614	\\
$w_9$	&	0.0860	&	$w_{19}$	&	0.0597	\\
$w_{10}$	&	0.0825	&	$w_{20}$	&	0.0579	\\[1ex] 
\hline   \hline   
\end{tabular} 
\label{tab3}  
\end{table} 

Table~\ref{tab4} reports the variables belonging to the optimal path-step and the results of the Kraskov test for each of them. After removing the variables that are not statistically significant for ``Calories", the final subset of variables turns out to include``Distance", ``Fat Burned" and ``Total Steps". Figure~\ref{fig6} shows the empirical densities of the variable ``Calories'' conditioned on each of these variables, together with  the density obtained of `Calories'' once this variable is conditioned on the set of the significant variables included in  $w_f$. The latter density, denoted by  $f(Calories|w_f)$, is computed as in Eq.~\ref{fy}.  
`
Figure~\ref{fig7} compares the  empirical density of ``Calories'' together with $f(Calories|w_f)$.
\begin{table}[ht] 
\caption{\citet{kraskov2004estimating}  mutual information test of independence between ``Calories" and $X_{w_2}$. The null of the test assumes independence between each variable in $w_2$ and ``Calories". } 
\centering        
\begin{tabular}{c c }  
Variables at  $w_2$ &   $p-$value \\  [0.5ex] 
\hline  \hline
Total.Steps	&	0.01	\\
Distance	&	0.01	\\
Fat.Burned	&	0.01	\\
Steps.9AM	&	1.00	\\
Steps.5PM	&	1.00	\\
Steps.6PM	&	0.94	\\
Aerobic.Walking.Time	&	1.00	\\
Steps.10AM	&	0.84	\\
Steps.3PM	&	1.00	\\[1ex] 
\hline   \hline   
\end{tabular} 
\label{tab4}  
\end{table}

\begin{figure}[H]\par\medskip
    \centering
    \includegraphics[scale=0.65]{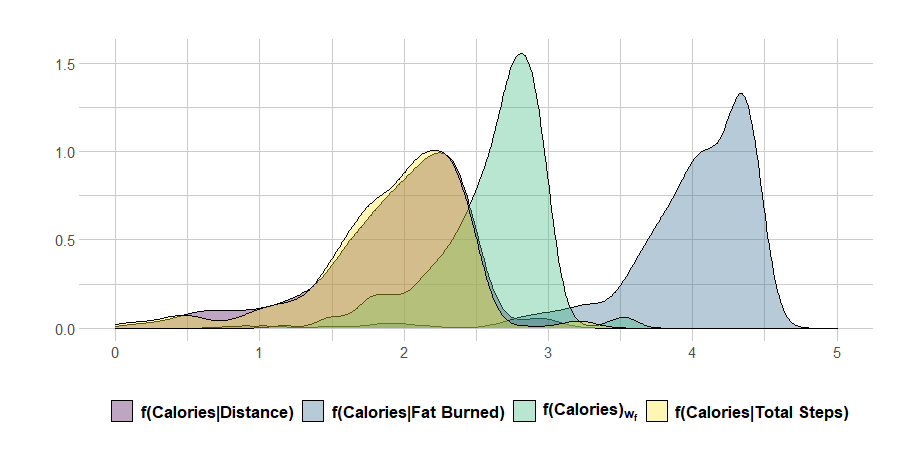}
    \caption{Empirical densities of  ``Calories'' conditioned on  each variables belonging to $w_f$ together with  $f(Calories|w_f)$}
    \label{fig6}
\end{figure} 
\begin{figure}[H]\par\medskip
    \centering
    \includegraphics[scale=0.65]{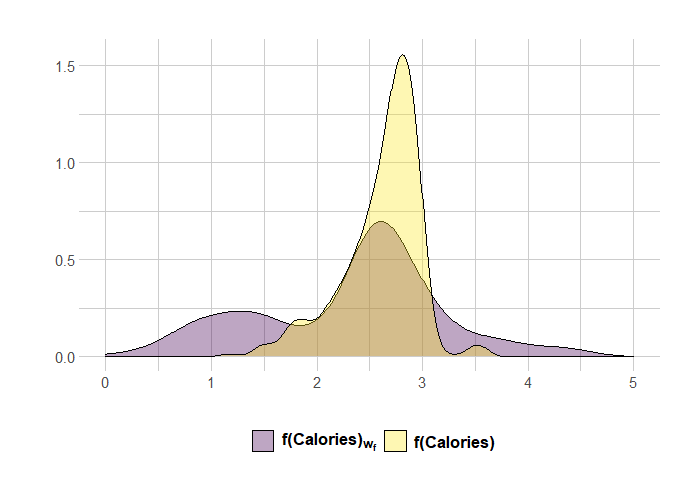}
    \caption{ Empirical density of ``Calories'' and  $f(Calories|w_f)$ } 
    \label{fig7}
\end{figure}
\subsection{LSVT Voice Rehabilitation  data-set}
The third application refers to the data-set built by Athanasios of the University of Oxford. It includes clinical information from speech signals  provided by LSVT Global, a company specialised in voice rehabilitation. The original study used 309 algorithms to characterize 126 speech signals from 14 people,
to determine the most parsimonious feature subset of variables needed 
to predict the binary response indicating the phonation quality during rehabilitation (acceptable vs unacceptable). Both cross-validation (10-fold cross-validation with 100 repetitions for statistical confidence) and leave one subject out methods were used for the validation of the findings. In both cases, they demonstrated a near 90\% accurate replication of the clinicians' assessment \citep{tsanas2013objective, tsanas2011nonlinear}. In detail, the data-set is composed of 126 observations and 314 variables of which 310 are features coming from the algorithms employed to characterize speech signals,  3 are  general demographics (participant id, age and gender) and one represinting the target variable ``Binary Response". The latter assumes the value 1 if the phonation quality is acceptable and 0 otherwise.\\ 
The BPA based on the EC index is here employed to find the best appropriate subset of variables that is capable of explaining the variable "Binary Response", according to the multiple features available in the data-set. The aim of this application is twofold. First,  the BPA  is here applied to prove that it works also when the variable of interest is  dichotomous.
Second, it is proved that the algorithm based on the EC index does not depend on the number of predictors and, in particular, it works even when the latter is higher than the sample size. \\ Figure~\ref{fig8} shows the conditional relationships among the features of the data-set identified by the algorithm. The optimal forest showed in Figure~\ref{fig8} consists of two components, a big one including the node of interest and another one composed of eight node representing the general demographics variables. 
After computed distances and path-steps starting from the ``Binary Response'' node, it results that  fifty-nine are the possible path-steps to explore in order to find the one including the variables with the most explanatory power for the ``Binary Response''.  
\begin{figure}[H]\par\medskip
\centering
\includegraphics[scale=0.65]{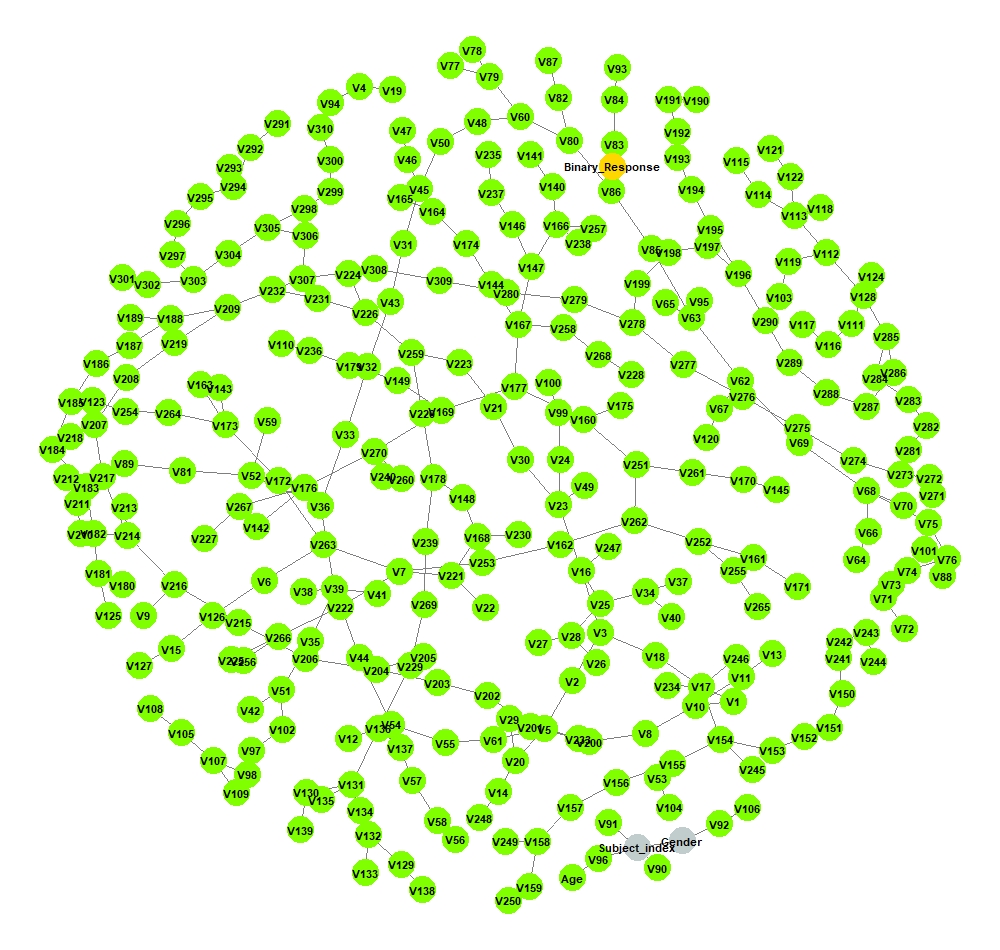}
\caption{The minimal BIC forest for the LSVT Voice Rehabilitation data-set. Continuous variable are green while the variable of interest ``Binary Response" is represented by the yellow node, while factors variables are
\label{fig8}
the grey nodes.}
\end{figure}

\begin{table}[ht]
\scalebox{0.85}{
\caption{path-steps' Entropy Coefficient for ``Binary Response'' in LSVT Voice Rehabilitation } 
\label{tab5}  
\centering        
\begin{tabular}{c c c c c c c c }  
path- &  Entropy & path- &  Entropy & path- &  Entropy & path- &  Entropy \\
steps $w_i$ & Coefficient (EC) & steps $w_i$ & Coefficient (EC) & steps $w_i$ & Coefficient (EC) & steps $w_i$ & Coefficient (EC) \\  [0.5ex]
\hline  \hline
$w_{1}$	&	-	&	$w_{16}$	&	0.025	&	$w_{31}$	&	0.025	&	$w_{46}$	&	0.021		\\
$w_{2}$	&	0.031	&	$w_{17}$	&	0.025	&	$w_{32}$	&	0.024	&	$w_{47}$	&	0.021		\\
$w_{3}$	&	0.033	&	$w_{18}$	&	0.025	&	$w_{33}$	&	0.024	&	$w_{48}$	&	0.021		\\
$w_{4}$	&	0.032	&	$w_{19}$	&	0.025	&	$w_{34}$	&	0.023	&	$w_{49}$	&	0.021	\\
$w_{5}$	&	0.031	&	$w_{20}$	&	0.025	&	$w_{35}$	&	0.023	&	$w_{50}$	&	0.021		\\
$w_{6}$	&	0.030	&	$w_{21}$	&	0.025	&	$w_{36}$	&	0.023	&	$w_{51}$	&	0.021		\\
$w_{7}$	&	0.028	&	$w_{22}$	&	0.025	&	$w_{37}$	&	0.023	&	$w_{52}$	&	0.021		\\
$w_{8}$	&	0.027	&	$w_{23}$	&	0.026	&	$w_{38}$	&	0.022	&	$w_{53}$	&	0.021		\\
$w_{9}$	&	0.027	&	$w_{24}$	&	0.026	&	$w_{39}$	&	0.022	&	$w_{54}$	&	0.021		\\
$w_{10}$ &	0.026	&	$w_{25}$	&	0.026	&	$w_{40}$	&	0.022	&	$w_{55}$	&	0.021		\\
$w_{11}$ &	0.026	&	$w_{26}$	&	0.026	&	$w_{41}$	&	0.021	&	$w_{56}$	&	0.021		\\
$w_{12}$ &	0.027	&	$w_{27}$	&	0.026	&	$w_{42}$	&	0.021	&	$w_{57}$	&	0.021		\\
$w_{13}$ &	0.025	&	$w_{28}$	&	0.025	&	$w_{43}$	&	0.021	&	$w_{58}$	&	0.021		\\
$w_{14}$ &	0.026	&	$w_{29}$	&	0.025	&	$w_{44}$	&	0.021	&	$w_{59}$	&	0.021		\\
$w_{15}$ &	0.025	&	$w_{30}$	&	0.025	&	$w_{45}$	&	0.021	&			&			\\[1ex] 

\hline   \hline   
\end{tabular}
}

\end{table} 

Table~\ref{tab5}, which provides the Entropy Coefficients for each path-step, identifies $w_3$ as the best path-step given that it has the highest EC index. Table~\ref{tab6} reports the variables belonging to the optimal path-step and shows the result of Kraskov's test, which assesses the independence of each variable belonging to $w_5$ with the variable ``Binary Response". 

\begin{table}[H]\par\medskip
\caption{ Kraskov's mutual information test of independence between ``Binary Response"" and each variable included in  $w_3$.}
\centering        
\begin{tabular}{c c }  
Variables at  $w_5$ &   $p-$value \\  [0.5ex] 
\hline  \hline
$V_{86}$ & 1.00 \\
$V_{80}$  & 0.02 \\ 
$V_{83}$ & 1.00 \\ 
$V_{85}$ & 0.82 \\ 
$V_{60}$ & 0.01 \\
$V_{82}$ & 0.02 \\ 
$V_{84}$ & 1.00 \\
$V_{63}$ & 0.98 \\[1ex] 
\hline   \hline   
\end{tabular} 
\label{tab6}  
\end{table} 

Figure~\ref{fig9} depicts the empirical densities of the variable ``Binary Response'' conditioned on each variable included in $w_{f}$ the set which, according to the Kraskov's test, proves to be significant for its explanation, together with the `Binary Response" density computed as in Eq.~\ref{fy} ( $f(\text{Binary Response}|w_f)$). Figure~\ref{fig10} compares the ```Binary Response'' empirical density, where 1 corresponds to ``acceptable" and 2 is ``unacceptable",  with the  $f(\text{Binary Response}|w_f)$,  where  the class  "acceptable" is the set of values for which $f(\text{Binary Response}|w_f) \leq  0.5$ and the class ``unacceptable" the set for which $f(\text{Binary Response}|w_f)>0.5$. 

\begin{figure}[H]\par\medskip
    \centering
    \includegraphics[scale=0.65]{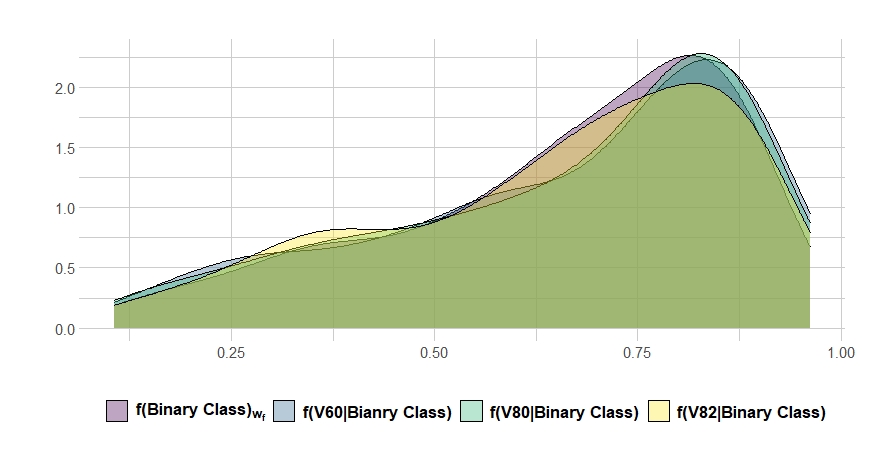}
    \caption{Empirical densities of  ``Binary Response'' conditioned on  each variables belonging to $w_f$ together with   $f(\textit{\text{Binary Response}}|w_f)$}
    \label{fig9}
\end{figure} 
\begin{figure}[H]\par\medskip
    \centering
    \includegraphics[scale=0.65]{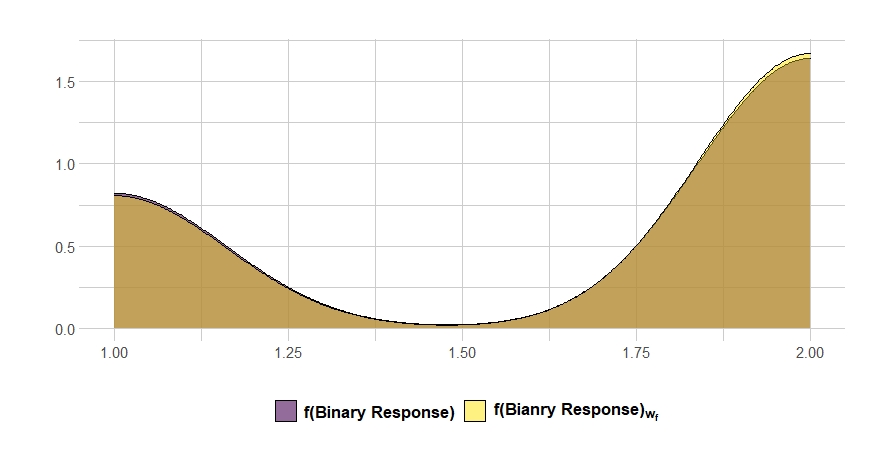}
    \caption{ Empirical density of ```Binary Response'' and  $f(\textit{\text{Binary Response}}|w_f)$ }
    \label{fig10}
\end{figure}
\subsection{Prostate cancer data-set}
The fourth data-set comes from the \citet{stamey1989prostate} study of prostate cancer and includes a set of variables useful to investigate the association between a prostate-specific antigen (PSA) and a number of prognostic clinical measurements useful to detect advanced prostate cancer.
The data-set is composed of nine variables observed on $97$ patients. The predictors, which are commonly considered \citep{stamey1989prostate}, are eight clinical measures: ``lcavol''(log of cancer volume, measured in milliliters),     ``lweight'' (log of prostate weight,measured in grams), ``age'' (age of the patient in years), ``lbph'' (log of the amount of benign prostatic hyperplasia, reported in $cm^2$), ``svi'' (presence or absence of seminal vesicle invasion), ``lcp'' (log of the capsular penetration,in cm), ``gleason'' (gleason score, a measure of the degree of aggressiveness of the tumor), and ``pg45'' (percentage of Gleason scores). The variable of interest is ``lpsa'' and it represents the log of PSA measured in ng/m. 

\begin{figure}[H]\par\medskip
\centering
\includegraphics[scale=0.40]{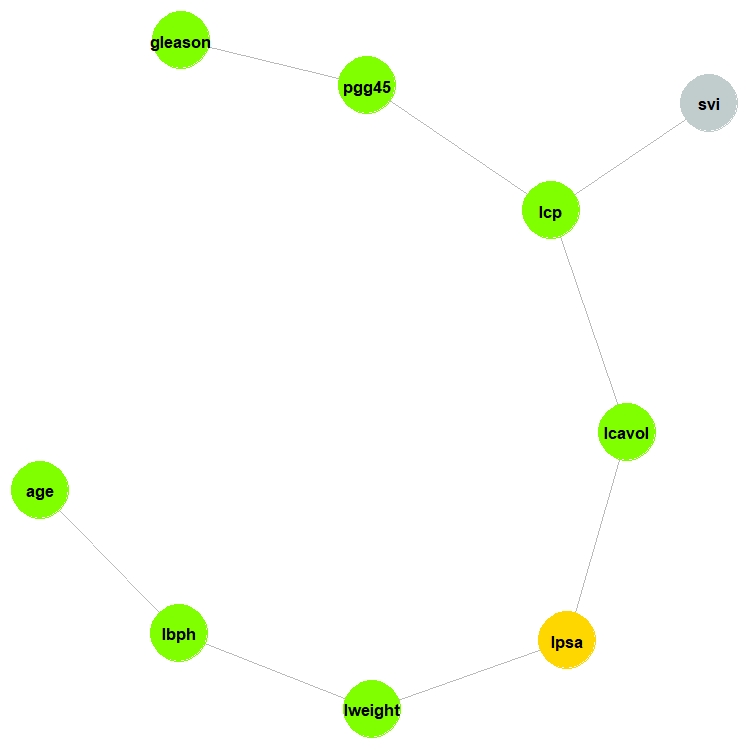}
\caption{The minimal BIC tree for the prostate cancer data-set shows the discrete variable ``svi'' with a grey node, the continuous variables in green and the variable of interest ``lpsa'' in yellow.} 
\label{fig11}
\end{figure}

Figure~\ref{fig11} shows the dependence relationships among the variables of the prostate cancer data-set through a minimal BIC tree. The node of interest ``lpsa'', is connected directly with the node ``lcavol'' and ``lweiht''. Using Definition~\ref{path-s}, $4$ path-steps can be identified. They are shown in  Table~\ref{tab3} together with $\bar{R}^2$ and MSE that would result if the variables included in each path-step were employed as regressors in a linear model. According to the results shown in Table~\ref{tab7}, $w_3$ turns out to be the best path-step. The variables included in $w_3$ are: ``lcavol'', ``lweight'', ``lcp", ``lbph", ``svi", ``pgg45" and ``age". 
The statistical significance analysis of these variables  leads to conclude that the best predictors for ``lpsa'' are the ones reported in Table~\ref{tab8}.  
Figure \ref{fig12} shows the empirical density of "lpsa" together with the one estimated from the linear model having as regressors the significant variables included in the optimal path-step, denoted by "est. dens lpsa" hereafter.

\begin{table}[htbp!] 
\caption{$\bar{R}^2$ and MSE for each path-steps for the ``lpsa'' variable of interest}\label{tab7}    
\centering        
\begin{tabular}{l c c }  
path-steps $w_i$ &   $\bar{R}^2$ & MSE \\  [0.5ex] 
\hline  \hline
$w_1$ & 0.64 & 0.73\\    
$w_2$ & 0.63 & 0.74\\ 
$w_3$ & 0.67 & 0.70\\ 
$w_4$ & 0.66 & 0.71\\  
\hline   \hline   
\end{tabular} 
\end{table}

\begin{table}[htbp!]
\caption{OLS regression for  ``lpsa'' with the selected predictors }
\label{tab8}
\begin{tabular}{lccccc}
Coefficients 	 	&	Estimate	&	Std. Error	&	t-value	&	$Pr(>|t|)$	&	Signif.	\\[0.5ex]
\hline	\hline				
$\beta_0$		&	-0.777	&	0.623	&	-1.25	&	0.215	&		\\
$\beta_{lcavol}$		&	0.526	&	0.075	&	7.02	&	3.5e-10	&	***	\\
$\beta_{lweight}$		&	0.662	&	0.176	&	3.77	&	0.00029	&	***	\\
$\beta_{svi=1}$		&	0.666	&	0.207	& 3.21	&	0.00180	&	**	\\
\multicolumn{6}{c}{$\bar{R}^2=0.636$;   Signif.   Codes: 0 ‘***’ 0.001 ‘**’   0.01 ‘*’ 0.05 ‘.’ 0.1}\\[0.1ex]  
\hline\hline 
\end{tabular}
\end{table}

\begin{figure}[H]\par\medskip
\centering
\includegraphics[scale=0.60]{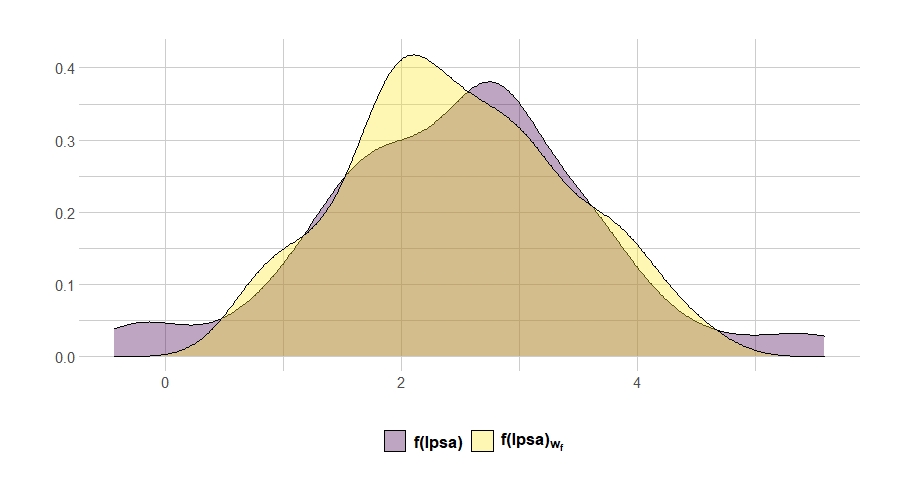}
\caption{Empirical density of ``lpsa'' and $f(lpsa|w_f)$}
\label{fig12}
\end{figure}

\subsection{Communities and Crime data-set}
 The fifth application employs the Communities and Crime data-set, that combines data from different sources \citep{redmond2002data}.
 As it contains more than $100$ variables, and precisely $128$ variables and $1994$ observations, a manual variable selection can go beyond the cognitive abilities of researchers and create fertile ground for scientific malpractices like reverse p-hacking \citep{head2015extent, chuard2019evidence}, and generating ex-post justified opaque models influenced by cognitive biases.
 After removing variables with too many missing data, the final data-set turns out to be composed of $102$ variables and $1994$ observations. \citet{redmond2002data} proposed a data driven approach to help police departments in developing a strategic approach for avoiding crimes. The target variable of this strategy is ``Violent crime per population'' that represents the total number of violent crimes per $100K$ population \citep{us20032000}. Using the BPA based on $\bar{R}^{2}$ we aim at selecting the appropriate explanatory variables to use as regressors in an econometric model explaining the variable ``Violent crime per population''. Figure~\ref{fig13} shows the minimal BIC forest built on the Communities and Crime data-set, while the variable label correspondence is shown in Table~\ref{NameVar}. The forest in Figure~\ref{fig13} is composed of one isolated nodes and one big component. The isolate node correspond to the variable ``fold", which denotes the number of non-random 10 fold cross-validation, potentially useful for debugging, paired tests. The node of interest ``Violent crime per population'' (node 102), belongs to the big component. This means that only the variables included in this component must be considered as its potential regressors in a linear model. The best path-step can be identified by simply computing $\bar{R}^2$ of  models including as regressors the variables of each path-step of the big component of the forest, starting from the node of interest.
\begin{figure}[H]\par\medskip
\centering
\includegraphics[scale=0.55]{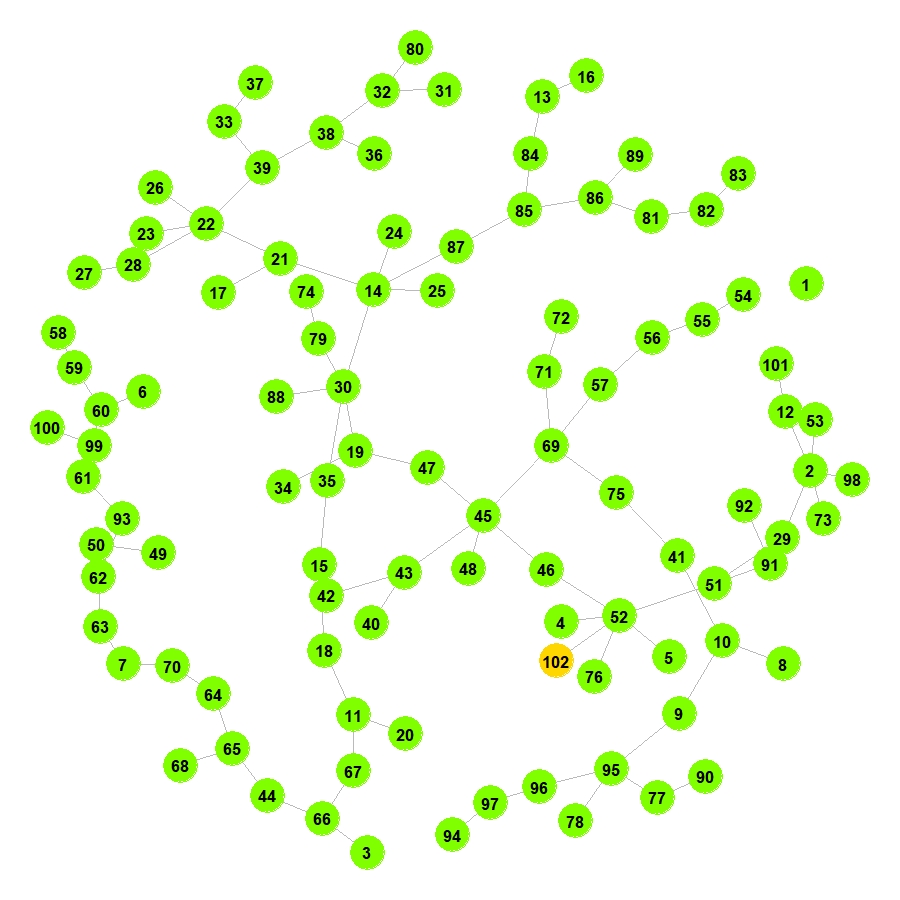}
\caption{The minimal BIC tree for the Communities and Crime data-set. 
Continuous variables are in green while the variable of interest ``Violent crime per Population'' is the yellow node. }
\label{fig13}
\end{figure}
According to Definition~\ref{path-s}, $24$ are the potential subsets of variables (path-steps) that can explain ``Violent crime per population". Table~\ref{tab3} shows both MSE and $\bar{R}^2$ of models that have the variables of each path-step as regressors. As we can see from Table~\ref{tab9}, the  algorithm identifies the path-step 20 as the best one. 
The final output of the algorithm, obtained by testing the statistical significance of the single variables of the optimal path-step,  is reported in Table~\ref{tab10}. As we can see the algorithm suggests using 23 variables to explain ``Violent crime per population". Some of them, such as
``PctKids2Par", the percentage of kids in family housing with two parents,
`` NumImmig", the total number of people known to be foreign born,
"PctPopUnderPov", the percentage of people under the poverty level
and
``PctWorkMom",  the percentage of moms of kids under 18 age, have a negative effect on ``Violent crime per population".
Others, such as the
``PctIlleg", the percentage of kids born to never married,
``PctVacantBoarded", the percent of vacant housing that is boarded up and
``NumStreet", the number of homeless people counted in the street have  a positive effect \footnote{Source: \href{https://archive.ics.uci.edu/ml/datasets/communities+and+crime}{https://archive.ics.uci.edu/ml/datasets/communities+and+crime}}
Figure ~\ref{fig14} shows the empirical density of the variable ``Violent crime per population" together with the density of the same once it is  estimated from the linear model whose regressors are
the significant variables of the best path-step, this density is denoted by $f( \text{Violent Crime per Population}|w_f)$ hereafter.

\begin{table}[H]\par\medskip
\centering
\begin{tabular}{l c c c c c}
path-steps $w_i$ &   $\bar{R}^2$ &  MSE  &  path-steps $w_i$ &   $\bar{R}^2$ & MSE \\[0.5ex]
\hline \hline
$w_1$	&	0.5474	&	0.1569	&	$w_{13}$	&	0.6544	&	0.1370	\\
$w_2$	&	0.6204	&	0.1438	&	$w_{14}$	&	0.6542	&	0.1370	\\
$w_3$	&	0.6234	&	0.1432	&	$w_{15}$	&	0.6554	&	0.1368	\\
$w_4$	&	0.6336	&	0.1413	&	$w_{16}$	&	0.6572	&	0.1364	\\
$w_5$	&	0.6447	&	0.1391	&	$w_{17}$	&	0.6567	&	0.1366	\\
$w_6$	&	0.6441	&	0.1392	&	$w_{18}$	&	0.6569	&	0.1365	\\
$w_7$	&	0.6429	&	0.1394	&	$w_{19}$	&	0.6571	&	0.1365	\\
$w_8$	&	0.6429	&	0.1394	&	$w_{20}$	&	0.6601	&	0.1359	\\
$w_9$	&	0.6453	&	0.1389	&	$w_{21}$	&	0.6599	&	0.1359	\\
$w_{10}$	&	0.6504	&	0.1379	&	$w_{22}$	&	0.6597	&	0.1360	\\
$w_{11}$	&	0.6564	&	0.1367	&	$w_{23}$	&	0.6599	&	0.1359	\\
$w_{12}$	&	0.6559	&	0.1367	&	$w_{24}$	&	0.6593	&	0.1360	\\
[1ex] 
\hline   \hline   
\end{tabular}
\caption{$\bar{R}^2$ and MSE for each path-steps considering as variable of interest  ``Violent crime per population''}
\label{tab9}
\end{table}

\begin{table}[H]\par\medskip
\begin{tabular}{lcccccc}

Coefficients 	&	Label Node 	&	Estimate	&	Std. Error	&	t-value	&	$Pr(>|t|)$	&	Signif.	\\[0.5ex]
\hline	\hline
$\beta_0$	&	-	&	0.459	&	0.075	&	6.083	&	1.42E-09	&	***	\\
$\beta_{PctIlleg}$	&	52	&	0.173	&	0.038	&	4.499	&	7.22e-06	&	***	\\
$\beta_{racepctblack}$	&	4	&	0.196	&	0.026	&	7.403	&	1.96e-13	&	***	\\
$\beta_{PctKids2Par}$	&	46	&	-0.335	&	0.063	&	-5.354	&	9.59e-08	&	***	\\
$\beta_{PctVacantBoarded}$	&	76	&	0.039	&	0.018	&	2.095	&	0.036313	&	*	\\
$\beta_{NumStreet}$	&	92	&	0.210	&	0.043	&	4.882	&	1.14e-06	&	***	\\
$\beta_{MalePctDivorce}$	&	40	&	0.113	&	0.032	&	3.498	&	0.000479	&	***	\\
$\beta_{NumImmig}$	&	53	&	-0.181	&	0.058	&	-3.118	&	0.001849	&	**	\\
$\beta_{HousVacant}$	&	73	&	0.178	&	0.030	&	6.043	&	1.80e-09	&	***	\\
$\beta_{PctPopUnderPov}$	&	30	&	-0.124	&	0.038	&	-3.235	&	0.001235	&	**	\\
$\beta_{PctEmploy}$	&	35	&	0.103	&	0.053	&	1.96	&	0.050167	&	.	\\
$\beta_{MedRent}$	&	87	&	0.285	&	0.054	&	5.279	&	1.44e-07	&	***	\\
$\beta_{pctWWage}$	&	15	&	-0.213	&	0.044	&	-4.79	&	1.79e-06	&	***	\\
$\beta_{whitePerCap}$	&	23	&	-0.103	&	0.032	&	-3.226	&	0.001278	&	**	\\
$\beta_{RentLowQ}$	&	84	&	-0.251	&	0.051	&	-4.919	&	9.41e-07	&	***	\\
$\beta_{OtherPerCap}$	&	27	&	0.000	&	0.000	&	2.495	&	0.012665	&	*	\\
$\beta_{pctUrban}$	&	13	&	0.038	&	0.009	&	4.387	&	1.21e-05	&	***	\\
$\beta_{pctWRetire}$	&	20	&	-0.100	&	0.030	&	-3.351	&	0.00082	&	***	\\
$\beta_{MedOwnCostPctIncNoMtg}$	&	90	&	-0.084	&	0.019	&	-4.478	&	7.98e-06	&	***	\\
$\beta_{pctWFarmSelf}$	&	16	&	0.034	&	0.019	&	1.81	&	0.070391	&	.	\\
$\beta_{PctPersDenseHous}$	&	70	&	0.227	&	0.042	&	5.389	&	7.94e-08	&	***	\\
$\beta_{PctNotSpeakEnglWell}$	&	63	&	-0.116	&	0.047	&	-2.444	&	0.014626	&	*	\\
$\beta_{PctWorkMom}$	&	50	&	-0.119	&	0.024	&	-5.028	&	5.42e-07	&	***	\\
$\beta_{PctForeignBorn}$	&	93	&	0.123	&	0.038	&	3.285	&	0.001037	&	**	\\
 		\multicolumn{7}{c}{$\bar{R}^2=0.673$;   Signif.   Codes: 0 ‘***’ 0.001 ‘**’   0.01 ‘*’ 0.05 ‘.’ 0.1}\\ [0.1ex]
\hline\hline
\end{tabular}
\caption{OLS regression for ``Violent crime per population'' with selected predictors}
\label{tab10}
\end{table}

\begin{figure}[H]\par\medskip
\centering
\includegraphics[scale=0.60]{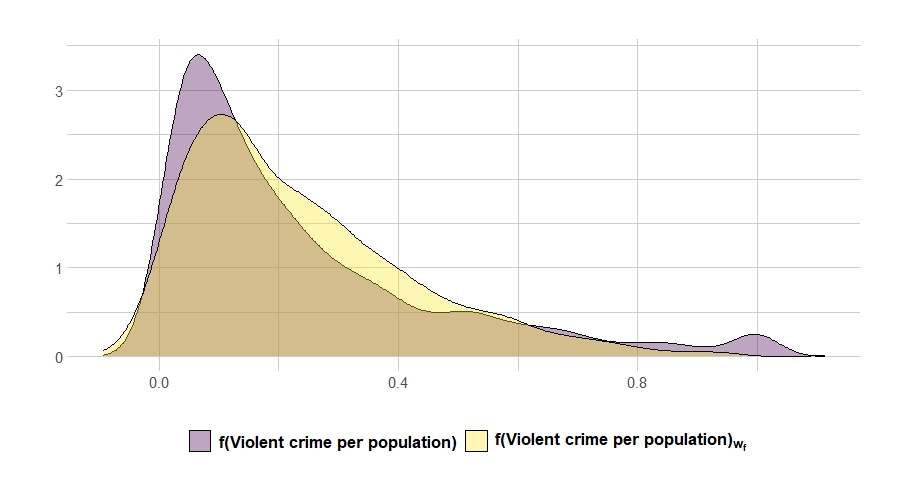}
\caption{Empirical density of ``Violent Crime per Population'' and $f( \textit{\text{Violent Crime per Population}}|w_f)$}
\label{fig14}
\end{figure}
\subsection{Blog-Feedback data-set}
The last application regards the Blog-Feedback data-set proposed by \citet{buza2014feedback}. This data-set is composed of $281$ features, including the target variable ``number of comments'' in the next 24 hours, and $52397$ observations. 
It includes also 60 test data-sets which can be employed to evaluate the goodness of models fitted to data\footnote{\href{https://archive.ics.uci.edu/ml/datasets/BlogFeedback}{https://archive.ics.uci.edu/ml/datasets/BlogFeedback}}. 
The data-set was created with the aim to simulate a real-world scenario where a model is trained to make predictions for the blog documents  published by using blog documents of the past \citep{buza2014feedback}. The specification of the variables is reported in Table~\ref{NameVar2}.

\begin{figure}[H]\par\medskip
\centering
\includegraphics[scale=0.60]{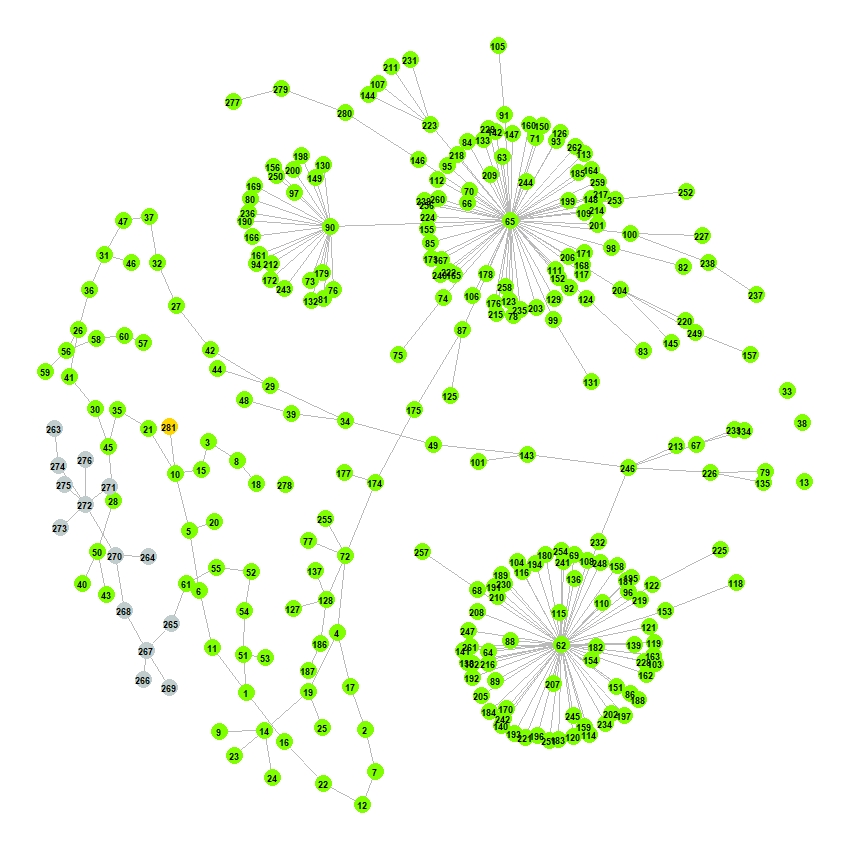}
\caption{The minimal BIC forest for the Blog-Feedback data-set. Discrete variables are shown in grey, continuous variables are in green, while the variable of interest ``number of comments'' is the yellow node.}
\label{fig15}
\end{figure}
Figure~\ref{fig15} represents the forest encoding the dependence relationships among the features of the Blog-Feedback data-set. This forest is composed of three isolated nodes (13, 38 and 278) and one big component including the target variable ``number of comments''. Following the Definition~\ref{path-s}, 23 path-steps can be identified. The latter are shown in Table~\ref{tab11} together with the $\bar{R}^2$ and MSEs of the linear models whose regressors are the variables of each potential path-step. The optimal path-step turns out to be $w_{10}$. 
After checking the statistical significance of each variable belonging to this path-step, the final set of best predictors turns out to be the one reported in in Table~\ref{tab12}.
Figure ~\ref{fig16} shows the empirical density of the variable \textit"number of Comments" together with the density of the latter estimated from a linear regression model whose regressors are the significant variables of the optimal path-step. The latter is denoted by $f(\text{ number of comments}|w_f)$ hereafter.

\begin{table}[H]\par\medskip
\centering
\begin{tabular}{
l c c c c c}
path-steps $w_i$ &   $\bar{R}^2$ & MSE  &  path-steps $w_i$ &   $\bar{R}^2$ & MSE \\[0.5ex]
\hline \hline
$w_1$	&	0.261	&	32.5	&	$w_{13}$	&	0.352	&	30.3	\\
$w_2$	&	0.264	&	32.4	&	$w_{14}$	&	0.351	&	30.4	\\
$w_3$	&	0.264	&	32.4	&	$w_{15}$	&	0.351	&	30.4	\\
$w_4$	&	0.264	&	32.4	&	$w_{16}$	&	0.351	&	30.4	\\
$w_5$	&	0.265	&	32.4	&	$w_{17}$	&	0.351	&	30.4	\\
$w_6$	&	0.266	&	32.3	&	$w_{18}$	&	0.351	&	30.4	\\
$w_7$	&	0.291	&	31.7	&	$w_{19}$	&	0.351	&	30.4	\\
$w_8$	&	0.347	&	30.5	&	 $w_{20}$	&	0.351	&	30.4	\\
$w_ 9$	&	0.347	&	30.5	&	$w_{21}$	&	0.351	&	30.4	\\
$w_{10}$	&	0.353	&	30.3	&	$w_{22}$	&	0.350	&	30.4	\\
$w_{11}$	&	0.352	&	30.3	&	$w_{23}$	&	0.350	&	30.4	\\
$w_{12}$	&	0.352	&	30.3	&		&		&		
	\\
[1ex] 
\hline   \hline   
\end{tabular}
\caption{$\bar{R}^2$ andMSE for each \textit{path-steps}; variable of interest: ``number of comment''}
\label{tab11}
\end{table}
\begin{table}[H]\par\medskip
\begin{tabular}{lcccccc}
Coefficients 	& 	Label Node &	Estimate	&	Std. Error	&	t value	&	$Pr(>|t|)$	&	Signif.\\[0.5ex]
\hline \hline
$\beta_0$	& -	&0.0386	&	0.1852	&	0.21	&	0.83	&		\\
$\beta_{6}$	& 6	&0.8184	&	0.017	&	48.23	&	$<$2e-16	&	***	\\
$\beta_{3}$	& 3	&-0.0976	&	0.0211	&	-4.62	&	3.80e-06	&	***	\\
$\beta_{22}$& 22 &	-1.0357	&	0.0637	&	-16.27	&	$<$2e-16	&	***	\\
$\beta_{26}$& 26 &	-19.9003	&	2.7327	&	-7.28	&	3.30e-13	&	***	\\
$\beta_{12}$& 12	&	1.0044	&	0.1311	&	7.66	&	1.80e-14	&	***	\\
$\beta_{52}$& 52	&	0.0717	&	0.0044	&	16.31	&	$<$2e-16	&	***	\\
$\beta_{36}$& 36 	&	58.9165	&	7.8038	&	7.55	&	4.40e-14	&	***	\\
$\beta_{56}$& 56	&	-1.3293	&	0.1426	&	-9.32	&	$<$2e-16	&	***	\\
$\beta_{7}$	& 7	&0.5189	&	0.1115	&	4.66	&	3.20e-06	&	***	\\
$\beta_{55}$& 55	&	0.117	&	0.0032	&	36.59	&	$<$2e-16	&	***	\\
$\beta_{58}$& 58	&	1.6097	&	0.236	&	6.82	&	9.10e-12	&	***	\\
$\beta_{2}$	& 2	&-0.195	&	0.0242	&	-8.06	&	8.10e-16	&	***	\\
\multicolumn{6}{c}{$\bar{R}^2=0.338$;   Signif.   Codes: 0 ‘***’ 0.001 ‘**’   0.01 ‘*’ 0.05 ‘.’ 0.1}\\[0.1ex]  
\hline\hline
\end{tabular}
\caption{Final output; ``number of comment'' model, \textit{OLS} regressions summary }
\label{tab12}
\end{table}

\begin{figure}[H]\par\medskip
\centering
\includegraphics[scale=0.60]{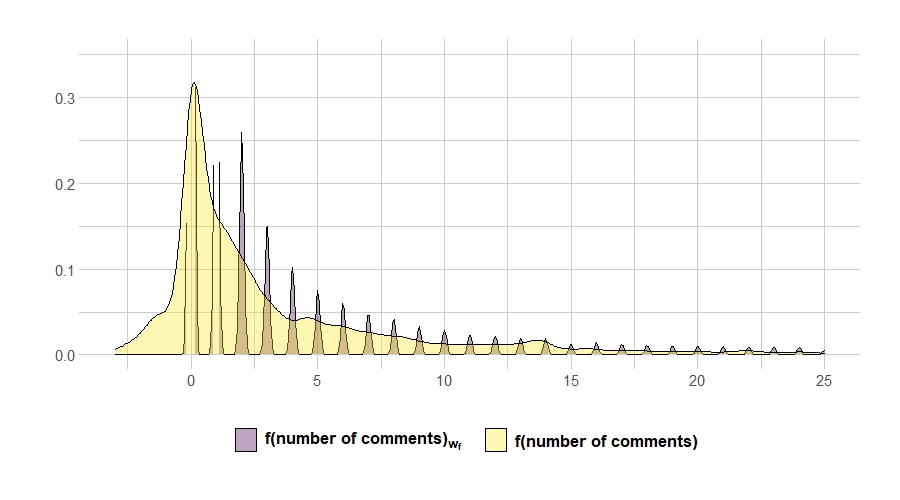}
\caption{Empirical density of ``number of comments''  and  $f(\textit{\text{ number of comments}}|w_f)$}
\label{fig16}
\end{figure}
\subsection{Variable Selection  Test}
The performance of the BPA, based on the EC index, has been evaluated by comparing the outcomes of the first three examples with those attainable by applying the minimum
redundancy maximum relevance (mRMRe) algorithm \citep{battiti1994using,kwak2002input,estevez2009normalized}. The latter is an heuristic approach based on information theory metrics, originally proposed by 
\cite{battiti1994using}, The algorithm selects the most relevant variables, either discrete or continuous, within a given set by using a penalising rule based on the amount of redundancy that the variables share with those previously selected by the algorithm \citep{kratzer2018varrank}. At each step, the variables that maximize a given score are considered relevant.   
The rational of this approach can be described as follows. Let $\bm{F}$ be a set of variables and $\bm{C}$ its subset of relevant variables ($\bm{C} \subseteq \bm{F}$). Moreover, let $\bm{S}$ denote the set of already selected variables and $x_i$ a candidate variable. The local score function for a mid scheme (Mutual Information Difference) can be expressed as \citep{kratzer2018varrank} \footnote{There are  two popular ways to combine relevance and redundancy: difference (mid) or quotient (miq)}: 
\begin{equation}\label{eq:Varank}
    g(\alpha,\bm{C},\bm{S},x_i)=
    \underbrace{I(x_i;\bm{C})}_{Relevance}-
    \sum_{x_i \in \bm{S}}\overbrace{\alpha(x_i,x_s,\bm{C},\bm{S})}^{Scaling factor}\underbrace{I(x_i;x_s)}_{Redundancy}
\end{equation}
This function heuristically estimates the variable ranks using mutual information and a  model that adopts multiple search schemes. Indeed, the normalizing function $\alpha$ can assume four possible values \footnote{These values define four implemented models which are: MIFS \citep{battiti1994using}, MIFS-U \citep{kwak2002input},mRMR \citep{peng2005feature} and  normalized MIFS\citep{estevez2009normalized}}. In this application, we have used the normalized MIFS  \citep{estevez2009normalized} to compute the normalizing function $\alpha$: 
\begin{equation*}
    \centering
    \alpha(x_i,x_s,\bm{C},\bm{S})=\frac{1}{|S|min(H(x_i,x_s))}
\end{equation*}
with $H$ denoting entropy, and the quotient (miq) to combine the relevance and redundancy. Then, a varrank analysis sequentially compares relevance with redundancy as defined in \eqref{eq:Varank}. The final output is an ordered list of the selected variables \footnote{The variables are listed either in decreasing or increasing order of importance depending on a backward or forward algorithm is implemented} and a matrix of scores.

  The varrank analysis has been applied to check if the variables selected by the BPA are relevant or not for the variables of interest. In the former case, the following should occur: 
  \begin{equation*}
      \sum_{f_i \in \bm{S}}\alpha(f_i,f_s,\bm{C},\bm{S})I(f_i;f_s)\approx0
  \end{equation*}
  Figure~\ref{fig17} shows the result of the varrank function \citep{kratzer2018varrank} that implement the  mRMRe algorithm. Negative scores indicate that redundancy (blue) prevails, while positive scores indicate that relevancy (red) prevails.
  The higher the relevance , the more pronounced the red colour, the higher the redundancy, the deeper the blue.  
  Looking at the graph, we can see that the selection operated by the BPA provides is appropriate  because all the selected variables turn out to be relevant. Indeed, all of them have received positive scores 
 highlighted by a more or less intensive red.
  \begin{figure}[H]\par\medskip
\centering
\begin{minipage}{.5\linewidth}
\centering
\subfloat[]{\label{vsa}\includegraphics[scale=.50]{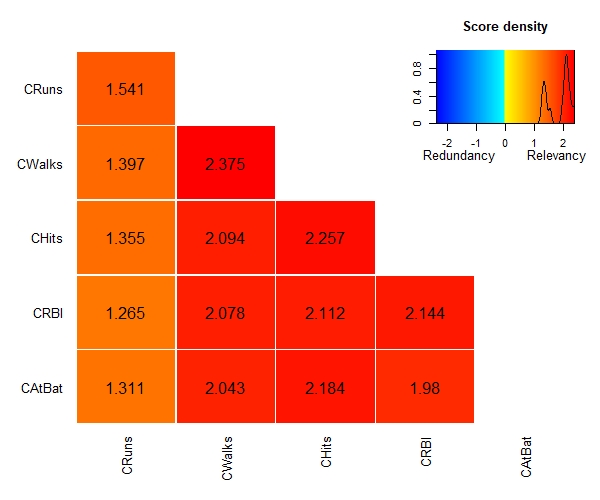}}
\end{minipage}%
\begin{minipage}{.5\linewidth}
\centering
\subfloat[]{\label{vsb}\includegraphics[scale=.50]{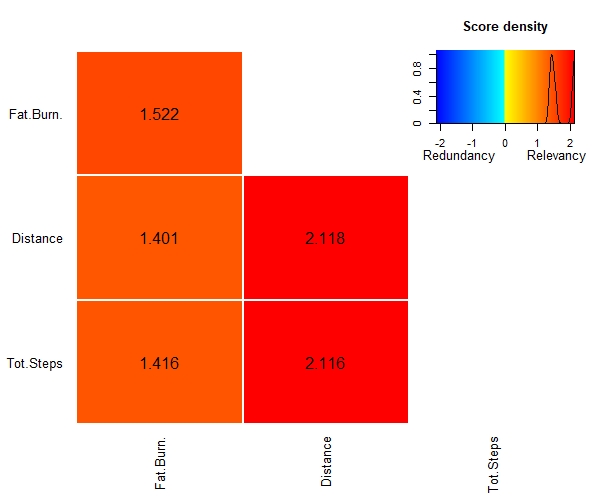}}
\end{minipage}\par\medskip
\centering
\subfloat[]{\label{vsc}\includegraphics[scale=.50]{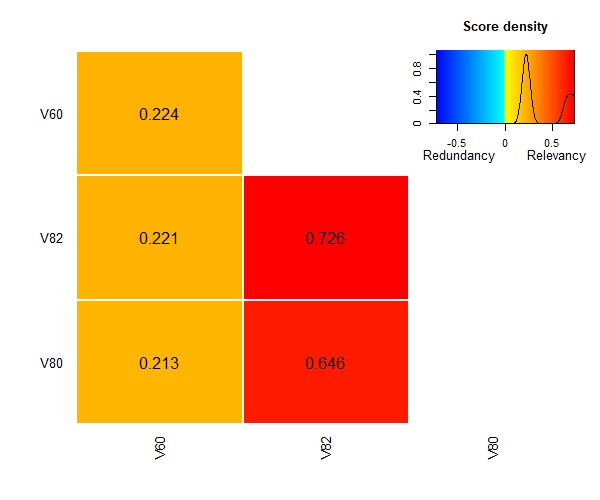}}
\caption{Output of the mRMRe algorithm for the variables of interest (a) ``Salary", (b) ``Calories", (c)  ``Binary Response" considered in the first three examples}
\label{fig17}
\end{figure}

\subsection{Prediction Test}
The performance of the BPA based on the $\bar{R}^2$, has been evaluated by comparing the outcomes of the last three examples with those obtainable by applying the elastic net \citep{zou2005regularization}. 
Like LASSO \citep{tibshirani1997lasso}, the elastic net simultaneously performs an automatic variable selection and continuous shrinkage. It can be 
seen as a regularized regression method that linearly combines the $L_1$ and $L_2$ penalties of the LASSO and ridge methods, respectively. 
In details, let us consider a data-set with $n$ observations and $p$ predictors. Let $\bm y= (y_1,\dots,y_n)^T$ be the observations on the target variable, and $\bm X=(\bm x_1,\dots,\bm x_p )$ the matrix of observations on $p$ predictors $\bm x_j=(x_{1j},\dots,x_{nj})^T$, $j=1,\dots,p$. By assuming that the both target variable and the predictors are standardized, then
the elastic net criterion, for any fixed non-negative $\lambda_1$ and $\lambda_2$,  is defined as follows (\citet{zou2005regularization})
\begin{equation}
    L(\lambda_1,\lambda_2, \boldsymbol \beta)=|\bm y-\bm X\boldsymbol \beta|^2+\lambda_1|\boldsymbol \beta|_1+\lambda_2|\boldsymbol \beta|_2
    \label{EL}
\end{equation}
where 
\begin{equation*}
    \centering
    \begin{split}
        |\boldsymbol\beta|^2=\sum_j^p \beta^2_j\\
        |\boldsymbol\beta|_1=\sum_j^p \beta_j
    \end{split}
\end{equation*}
The elastic net estimator $\hat{\boldsymbol{\beta}}$ is the solution of Eq.~\ref{EL} \citep{zou2005regularization}:
\begin{equation}
    \hat{\boldsymbol{\beta}}=\text{arg}\,\min_{\boldsymbol{\beta}} \{L(\lambda_1,\lambda_2,\boldsymbol{\beta})\}
\end{equation}
By setting $\alpha=\lambda_2/(\lambda_1+\lambda_2)$, the solution of Eq.~\ref{EL} turns out to be equivalent to that of the optimization problem 
\begin{equation}
    \hat{\boldsymbol{\beta}}=\text{arg}\,\min_{\boldsymbol{\beta}} |\bm{y}-\bm{X} \boldsymbol{\beta}|^2, \quad \text{subject to} \quad (1-\alpha)|\boldsymbol{\beta}|_1+\alpha|\boldsymbol{\beta}|^2 \leq t \quad  \text{for some t}
\end{equation}
The function $(1-\alpha)|\boldsymbol{\beta}|_1+\alpha|\boldsymbol{\beta}|^2$, called the elastic net penalty, is a convex combination of the LASSO and Ridge penalty \citep{zou2005regularization}.\\
The outcomes of the elastic net have been compared to those provided by the  best path-step algorithm, for each variable of interest of the last three examples. 
The data-sets Prostate cancer and Communities and Crime have been split into two parts: one ($70\%$ of the data-set) devoted to the train and the other ($30\%$ of the data-set) finalized to test the accuracy performance of both methods. The procedure has been repeated 100 times with the purpose of avoiding problems like over-fitting or selection bias. 
As for the Blog-Feedback, the 60 included test data-sets have been employed to evaluate the performance of the prediction.  Figure~\ref{main:a} shows that the ``lpsa"" model, built with regressors selected by using the BPA, has a MSE which is $64$ times lower than the one of the elastic net. As shown in Figure~\ref{main:b}, the model ``Violent Crime per population"", built with regressors selected by the BPA has a MSE which is $93$ times lower than the one of the elastic net method. Eventually, Figure~\ref{main:c} shows that the prediction performance of the BPA is very similar to that of the elastic net for the ``number of comment'' model. Nevertheless, the MSE of the BPA is 50\% of times lower than that of the elastic net.

\begin{figure}[H]\par\medskip
\centering
\begin{minipage}{.5\linewidth}
\centering
\subfloat[]{\label{main:a}\includegraphics[scale=.35]{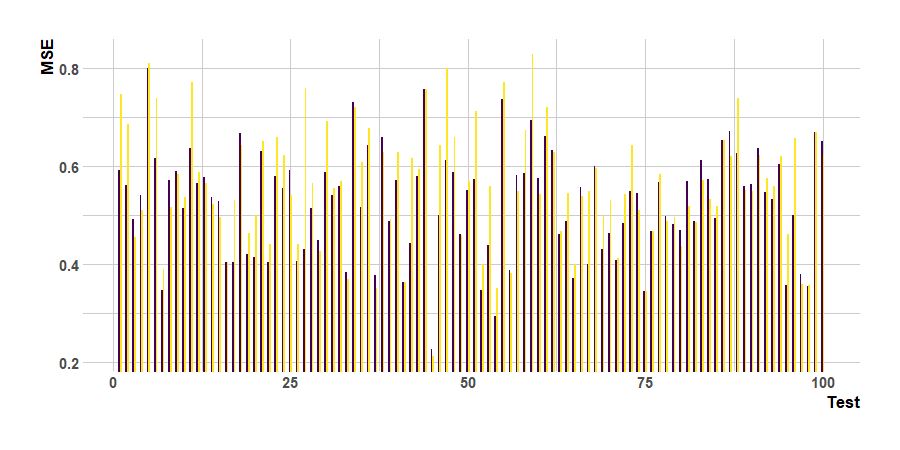}}
\end{minipage}%
\begin{minipage}{.5\linewidth}
\centering
\subfloat[]{\label{main:b}\includegraphics[scale=.35]{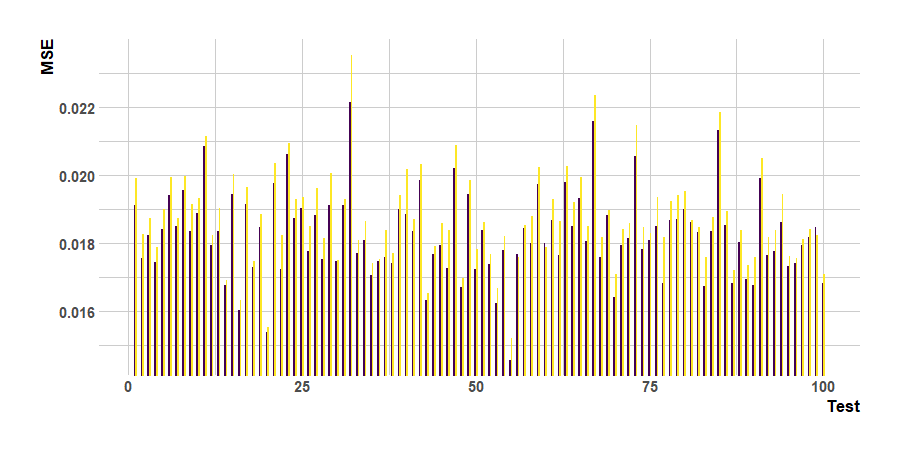}}
\end{minipage}\par\medskip
\centering
\subfloat[]{\label{main:c}\includegraphics[scale=.40]{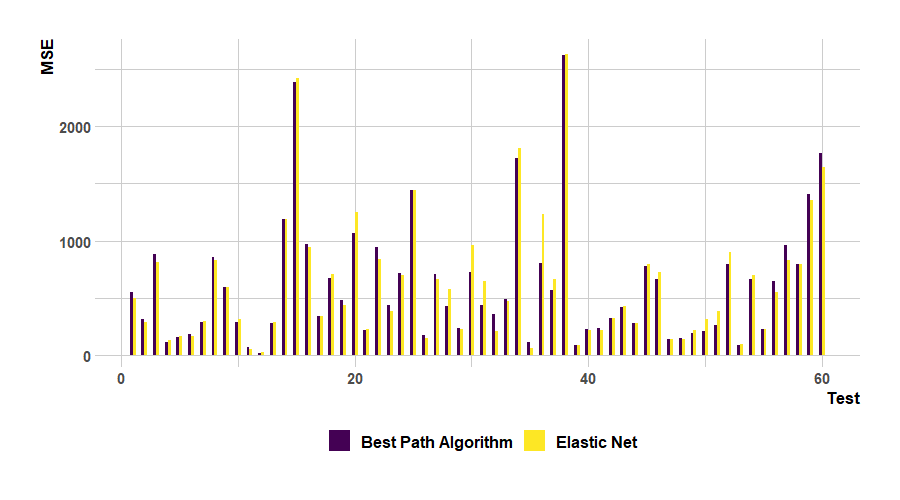}}
\caption{Comparison of the prediction performance between the BPA  and the elastic net: (a)  ``lpsa'' model, (b)  ``Violent crime per population'' model, (c) ``number of comment'' model}
\label{fig18}
\end{figure}

\section{Conclusion}\label{5}
Variable selection has become an arduous task in the last decade, due to the the availability of large data-sets.
Different strategies have been developed to mitigate the 'embarrassment of riches' in the choice of the variables \citep [among others]{eklund2007embarrassment, altman2018curse}. Most of them arise from machine or statistical learning \citep{qian2022feature,hu2022feature}. Many of machine and statistical learning models are black boxes that do not allow the identification and interpretation of the causal relationships among variables of the data-set at hand \citep{rudin2019stop}. 
The need for Machine learning techniques to maintain causal and interpretative aspects is ever increasing, especially when they are applied to some specific methodological areas, such as the economic and financial one.
This paper tries to meet this challenge as it proposes an automatic variable selection that combines the availability of large data-sets with the interpretability of statistical and econometric models.
Indeed, it proposes an automatic variable selection algorithm that hinges on Graphical Models  \citep{jordan2004graphical}, to detect the relevant explanatory variables for a given variable of interest. Once the minimal BIC forest for the data set at hand is built, the mutual information between the variables of each path-step and the node of interest is evaluated, by using the entropy coefficient (EC) of determination. The latter is a predictive power measure for generalized linear models (GLMs) which can be applied to continuous and polytomous variables.
In GLMs, this correlation coefficient is a measure of linearity of a response variable and the canonical parameters, and can be viewed as the proportion of the reduced
uncertainty of a response variable due to the explanatory variables. It can be computed via the symmetric Kullback–Leibler information that measures the similarity between the density of the response variable and the one obtained by conditioning the same on the independent ones (Eshima and Tabata, 2007).
The higher the EC index associated to a given path-step, the better the explanatory power of the variables included in the same. The EC index specializes into the standard coefficient of determination, when the variables under study can be read as potential regressors of a linear econometric model. Once the best path-step is determined, the subset of significant regressors is determined via suitable tests. The wide range of examples here presented, show the effectiveness of the proposed approach also in comparative terms.

Machine learning applied to economic problems therefore needs to retain the causal and interpretative aspects, in addition to the typical predictive orientation of this methodological area.
And despite there has been an increasing trend in healthcare \citep{wiens2018machine},
economic \citep[among others]{athey2019machine,roth2004generalized}
and biological \citep{camacho2018next}
to leverage machine learning or statistical learning  for high-stakes prediction applications that deeply impact human lives, many of the machine learning and statistical learning models are black boxes that do not explain their predictions in a way that humans can understand the results \citep{rudin2019stop}. 
In this way, we can combine the availability of large data-sets with the interpretability of the econometric model,  obtaining models with a god accuracy also in prediction terms. \\
\\

\bibliographystyle{abbrvnat}
\bibliography{literature}
\newpage

\begin{appendices}
\section{Variable Name}

\begin{table}[H]\par\medskip
\centering
\scalebox{0.85}{
\begin{tabular}{|c|c|c|c|}
\hline
\textbf{Label Node} & \textbf{Variable}            & \textbf{Label Node} & \textbf{Variable}              \\ \hline
\textit{1}	&	\textit{fold}	&	\textit{52}	&	\textit{PctIlleg}	\\	\hline
\textit{2}	&	\textit{population}	&	\textit{53}	&	\textit{NumImmig}	\\	\hline
\textit{3}	&	\textit{householdsize}	&	\textit{54}	&	\textit{PctImmigRecent}	\\	\hline
\textit{4}	&	\textit{racepctblack}	&	\textit{55}	&	\textit{PctImmigRec5}	\\	\hline
\textit{5}	&	\textit{racePctWhite}	&	\textit{56}	&	\textit{PctImmigRec8}	\\	\hline
\textit{6}	&	\textit{racePctAsian}	&	\textit{57}	&	\textit{PctImmigRec10}	\\	\hline
\textit{7}	&	\textit{racePctHisp}	&	\textit{58}	&	\textit{PctRecentImmig}	\\	\hline
\textit{8}	&	\textit{agePct12t21}	&	\textit{59}	&	\textit{PctRecImmig5}	\\	\hline
\textit{9}	&	\textit{agePct12t29}	&	\textit{60}	&	\textit{PctRecImmig8}	\\	\hline
\textit{10}	&	\textit{agePct16t24}	&	\textit{61}	&	\textit{PctRecImmig10}	\\	\hline
\textit{11}	&	\textit{agePct65up}	&	\textit{62}	&	\textit{PctSpeakEnglOnly}	\\	\hline
\textit{12}	&	\textit{numbUrban}	&	\textit{63}	&	\textit{PctNotSpeakEnglWell}	\\	\hline
\textit{13}	&	\textit{pctUrban}	&	\textit{64}	&	\textit{PctLargHouseFam}	\\	\hline
\textit{14}	&	\textit{medIncome}	&	\textit{65}	&	\textit{PctLargHouseOccup}	\\	\hline
\textit{15}	&	\textit{pctWWage}	&	\textit{66}	&	\textit{PersPerOccupHous}	\\	\hline
\textit{16}	&	\textit{pctWFarmSelf}	&	\textit{67}	&	\textit{PersPerOwnOccHous}	\\	\hline
\textit{17}	&	\textit{pctWInvInc}	&	\textit{68}	&	\textit{PersPerRentOccHous}	\\	\hline
\textit{18}	&	\textit{pctWSocSec}	&	\textit{69}	&	\textit{PctPersOwnOccup}	\\	\hline
\textit{19}	&	\textit{pctWPubAsst}	&	\textit{70}	&	\textit{PctPersDenseHous}	\\	\hline
\textit{20}	&	\textit{pctWRetire}	&	\textit{71}	&	\textit{PctHousLess3BR}	\\	\hline
\textit{21}	&	\textit{medFamInc}	&	\textit{72}	&	\textit{MedNumBR}	\\	\hline
\textit{22}	&	\textit{perCapInc}	&	\textit{73}	&	\textit{HousVacant}	\\	\hline
\textit{23}	&	\textit{whitePerCap}	&	\textit{74}	&	\textit{PctHousOccup}	\\	\hline
\textit{24}	&	\textit{blackPerCap}	&	\textit{75}	&	\textit{PctHousOwnOcc}	\\	\hline
\textit{25}	&	\textit{indianPerCap}	&	\textit{76}	&	\textit{PctVacantBoarded}	\\	\hline
\textit{26}	&	\textit{AsianPerCap}	&	\textit{77}	&	\textit{PctVacMore6Mos}	\\	\hline
\textit{27}	&	\textit{OtherPerCap}	&	\textit{78}	&	\textit{MedYrHousBuilt}	\\	\hline
\textit{28}	&	\textit{HispPerCap}	&	\textit{79}	&	\textit{PctHousNoPhone}	\\	\hline
\textit{29}	&	\textit{NumUnderPov}	&	\textit{80}	&	\textit{PctWOFullPlumb}	\\	\hline
\textit{30}	&	\textit{PctPopUnderPov}	&	\textit{81}	&	\textit{OwnOccLowQuart}	\\	\hline
\textit{31}	&	\textit{PctLess9thGrade}	&	\textit{82}	&	\textit{OwnOccMedVal}	\\	\hline
\textit{32}	&	\textit{PctNotHSGrad}	&	\textit{83}	&	\textit{OwnOccHiQuart}	\\	\hline
\textit{33}	&	\textit{PctBSorMore}	&	\textit{84}	&	\textit{RentLowQ}	\\	\hline
\textit{34}	&	\textit{PctUnemployed}	&	\textit{85}	&	\textit{RentMedian}	\\	\hline
\textit{35}	&	\textit{PctEmploy}	&	\textit{86}	&	\textit{RentHighQ}	\\	\hline
\textit{36}	&	\textit{PctEmplManu}	&	\textit{87}	&	\textit{MedRent}	\\	\hline
\textit{37}	&	\textit{PctEmplProfServ}	&	\textit{88}	&	\textit{MedRentPctHousInc}	\\	\hline
\textit{38}	&	\textit{PctOccupManu}	&	\textit{89}	&	\textit{MedOwnCostPctInc}	\\	\hline
\textit{39}	&	\textit{PctOccupMgmtProf}	&	\textit{90}	&	\textit{MedOwnCostPctIncNoMtg}	\\	\hline
\textit{40}	&	\textit{MalePctDivorce}	&	\textit{91}	&	\textit{NumInShelters}	\\	\hline
\textit{41}	&	\textit{MalePctNevMarr}	&	\textit{92}	&	\textit{NumStreet}	\\	\hline
\textit{42}	&	\textit{FemalePctDiv}	&	\textit{93}	&	\textit{PctForeignBorn}	\\	\hline
\textit{43}	&	\textit{TotalPctDiv}	&	\textit{94}	&	\textit{PctBornSameState}	\\	\hline
\textit{44}	&	\textit{PersPerFam}	&	\textit{95}	&	\textit{PctSameHouse85}	\\	\hline
\textit{45}	&	\textit{PctFam2Par}	&	\textit{96}	&	\textit{PctSameCity85}	\\	\hline
\textit{46}	&	\textit{PctKids2Par}	&	\textit{97}	&	\textit{PctSameState85}	\\	\hline
\textit{47}	&	\textit{PctYoungKids2Par}	&	\textit{98}	&	\textit{LandArea}	\\	\hline
\textit{48}	&	\textit{PctTeen2Par}	&	\textit{99}	&	\textit{PopDens}	\\	\hline
\textit{49}	&	\textit{PctWorkMomYoungKids}	&	\textit{100}	&	\textit{PctUsePubTrans}	\\	\hline
\textit{50}	&	\textit{PctWorkMom}	&	\textit{101}	&	\textit{LemasPctOfficDrugUn}	\\	\hline
\textit{51}	&	\textit{NumIlleg}	&	\textit{102}	&	\textit{ViolentCrimesPerPop}	\\	\hline

\end{tabular}
}
\caption{Name of the variables \textit{Communities and Crime} data-set}
\label{NameVar}
\end{table}

\begin{table}[H]\par\medskip
\begin{tabular}{|c|c|}
\hline
\textbf{Label} & \textbf{Variable specification}                                                                                                                                                                                                                        \\ \hline
1…50           & Average,   standard deviation, min, max and median of the \\                               & 51...60 for the source of the current blog post\\

&With source we mean the blog on which the post appeared                                                                                                                                              \\ \hline
51             & Total   number of comments before basetime                                                                                                                                                                                                             \\ \hline
52             & \begin{tabular}[c]{@{}c@{}}Number   of comments in the last 24 hours before the\\      basetime\end{tabular}                                                                                                                                           \\ \hline
53             & \begin{tabular}[c]{@{}c@{}}Let T1 denote   the datetime 48 hours before basetime,\\      Let T2 denote the datetime 24 hours before basetime.\\      This attribute is the number of comments in the time period\\      between T1 and T2\end{tabular} \\ \hline
54             & \begin{tabular}[c]{@{}c@{}}Number   of comments in the first 24 hours after the\\      publication of the blog post, but before basetime\end{tabular}                                                                                                  \\ \hline
55             & The   difference of Attribute 52 and Attribute 53                                                                                                                                                                                                      \\ \hline
56…60          & \begin{tabular}[c]{@{}c@{}}The   same features as the attributes 51\dots55, but\\      features 56\dots60 refer to the number of links (trackbacks),\\      while features 51\dots55 refer to the number of comments.\end{tabular}                           \\ \hline
61             & \begin{tabular}[c]{@{}c@{}}The length of time between the publication   of the blog post\\      and basetime\end{tabular}                                                                                                                              \\ \hline
62             & The length of the blog post                                                                                                                                                                                                                            \\ \hline
63\dots262       & \begin{tabular}[c]{@{}c@{}}The 200 bag of words features   for 200 frequent words of the\\      text of the blog post\end{tabular}                                                                                                                     \\ \hline
263…269        & \begin{tabular}[c]{@{}c@{}}Binary indicator features (0 or 1) for the   weekday\\      (Monday\dots Sunday) of the basetime\end{tabular}                                                                                                                  \\ \hline
270\dots276        & \begin{tabular}[c]{@{}c@{}}Binary indicator features (0 or 1) for the   weekday\\      (Monday\dots Sunday) of the date of publication of the blog\\      post\end{tabular}                                                                               \\ \hline
277            & \begin{tabular}[c]{@{}c@{}}Number of parent pages: we   consider a blog post P as a\\      parent of blog post B, if B is a reply (trackback) to\\      blog post P\end{tabular}                                                                       \\ \hline
278\dots280       & \begin{tabular}[c]{@{}c@{}}Minimum, maximum, average number   of comments that the\\      parents received\end{tabular}                                                                                                                                \\ \hline
281            & \begin{tabular}[c]{@{}c@{}}The target: the number of   comments in the next 24 hours\\      (relative to basetime)\end{tabular}                                                                                                                        \\ \hline
\end{tabular}
\caption{Name of the variables \textit{Blog-Feedback} data-set}
\label{NameVar2}
\end{table}

\section{Code Availability}
The analyses were performed using the R library
gRapHD which the authors have made available to the R community
via the CRAN repository \citep{JSSv037i01}
\end{appendices}
\end{document}